\documentclass[10pt,twocolumn,letterpaper]{article}
 \usepackage[pagenumbers]{cvpr} 

\usepackage{algorithm,algorithmicx,algpseudocode}
\usepackage{amsmath}
\usepackage{amssymb}
\usepackage{amsthm}
\usepackage{color}
\usepackage{enumitem}
\usepackage{float}
\setlist[itemize]{leftmargin=*}
\newtheorem{theorem}{Theorem}
\newtheorem{lemma}{Lemma}

\usepackage[dvipsnames]{xcolor}

\definecolor{cvprblue}{rgb}{0.21,0.49,0.74}
\usepackage[pagebackref,breaklinks,colorlinks,citecolor=cvprblue]{hyperref}

\newcommand{\myparatight}[1]{\noindent{\bf {#1}.}~}
\newcommand{\name}{\text{MMCert}}
\def\paperID{13216} 
\def\confName{CVPR}
\def\confYear{2024}

\title{{\name}: Provable Defense against Adversarial Attacks to Multi-modal Models}

\author{
 {Yanting Wang$^{1}$, Hongye Fu${^{2}}$\thanks{Hongye Fu performed this research when he was a remote intern.}, Wei Zou$^1$, and Jinyuan Jia$^{1}$} \\
 $^1$The Pennsylvania State University, $^2$Zhejiang University\\
 $^1$\{yanting, weizou, jinyuan\}@psu.edu, $^2$3200102866@zju.edu.cn}

\begin{document}
\maketitle
\begin{abstract}
Different from a unimodal model whose input is from a single modality, the input (called multi-modal input) of a multi-modal model is from multiple modalities such as image, 3D points, audio, text, etc. Similar to unimodal models, many existing studies show that a multi-modal model is also vulnerable to adversarial perturbation, where an attacker could add small perturbation to all modalities of a multi-modal input such that the multi-modal model makes incorrect predictions for it. Existing certified defenses are mostly designed for unimodal models, which achieve sub-optimal certified robustness guarantees when extended to multi-modal models as shown in our experimental results. In our work,  we propose {\name}, the first certified defense against adversarial attacks to a multi-modal model. We derive a lower bound on the performance of our {\name} under arbitrary adversarial attacks with bounded perturbations to both modalities (e.g., in the context of auto-driving, we bound the number of changed pixels in both RGB image and depth image). We evaluate our {\name} using two benchmark datasets: one for the multi-modal road segmentation task and the other for the multi-modal emotion recognition task. Moreover, we compare our {\name} with a state-of-the-art certified defense extended from unimodal models. Our experimental results show that our {\name} outperforms the baseline. 
\end{abstract}
\section{Introduction}
With the rapid advancement of machine learning, multi-modal models have emerged as a powerful paradigm. Differing from their unimodal counterpart whose input is from a singular modality, these multi-modal models leverage input (called \emph{multi-modal input}) from diverse modalities such as images, 3D data points, audio, and text~\cite{kazakos2019epic,prakash2021multi,chumachenko2022self,fan2020sne, wang2023vqa}. Those multi-modal models have been widely used in many security and safety critical applications such as autonomous driving~\cite{fan2020sne, li2022deepfusion, wang2021pointaugmenting,teichmann2018multinet,natan2022end} and medical imaging~\cite{guo2019deep}. 

 As shown in many existing studies~\cite{goodfellow2014explaining, nguyen2015deep, sharif2016accessorize}, unimodal models are susceptible to adversarial attacks. There is no exception for multi-modal models. In particular, many recent studies~\cite{cheng2023fusion,shen2020drift,  zhang2022towards,tian2021can, xu2018fooling, shah2019cycle} showed that multi-modal models are also vulnerable to adversarial perturbations. In particular, an attacker could simultaneously manipulate all modalities of a multi-modal input such that a multi-modal model makes incorrect predictions. For instance, in the scenario of road segmentation for auto-driving, the attacker can add small perturbations to both the RGB image (captured by a camera) and the depth image (captured by a LiDAR depth sensor) to degrade the segmentation quality. Similarly, in the scenario of video emotion recognition, the attacker can apply subtle disruptions to both visual and audio data to reduce prediction accuracy. 

Many defenses were proposed to defend against adversarial attacks, In particular, they can be categorized into \emph{empirical defenses}~\cite{shafahi2019adversarial, wong2020fast, madry2017towards, kim2019single,yang2021defending, tian2021can,wang2021certified} and \emph{certified defenses}~\cite{cohen2019certified,jia2019certified,gowal2018effectiveness,levine2019sparse,wong2018provable, chiang2020certified,ye2020safer,lecuyer2019certified,zeng2023certified,xiang2022patchcleanser}. Many existing studies~\cite{athalye2018obfuscated,carlini2017adversarial,uesato2018adversarial} showed that most empirical defenses could be broken by strong, adaptive attacks (one exception is adversarial training~\cite{madry2017towards}). Therefore, we focus on certified defense in this work. Existing certified defenses are mainly designed for unimodal models (its input is from a single modality). Our experimental results show that they achieve sub-optimal performance when extended to defend against adversarial attacks for multi-modal models. The key reason is that when the attacker adds $l_p$ bounded perturbations to all modalities, the space of perturbed multi-modal inputs cannot be simply formulated as a $l_p$ ball. In this work, we focus on $l_0$-like adversarial attacks applied to each modality (i.e., manipulate a certain number of features for each modality) due to their straightforward applicability across various modalities. The investigation of alternative forms of attacks is reserved for future research.

\myparatight{Our work} We propose {\name}, the first certified defense against adversarial attacks to multi-modal models. Suppose we have a multi-modal input $\mathbf{M}=(\mathbf{m}_1, \mathbf{m}_2, \cdots, \mathbf{m}_T)$ with $T$ modalities, where $\mathbf{m}_i$ contains a set/sequence of basic elements from the $i$-th modality. We consider a general scenario, where each element could be arbitrary. For instance, each element could be a pixel value, a 3D point, an image frame, an audio frame, etc.. Given a multi-modal input $\mathbf{M}$ and a multi-modal model $g$ (called \emph{base multi-modal model}), we first create multiple sub-sampled multi-modal inputs. In particular, each sub-sampled multi-modal input is obtained by randomly sub-sampling $k_1, k_2, \cdots, k_T$ basic elements from $\mathbf{m}_1, \mathbf{m}_2, \cdots, \mathbf{m}_T$, respectively. Then, we use the base multi-modal model $g$ to make a prediction for each sub-sampled multi-modal input. Finally, we build an ensemble multi-modal model by aggregating those predictions as the final prediction made by our ensemble multi-modal classifier for the given multi-modal input $\mathbf{M}$. 

We derive the provable robustness guarantee of our ensemble multi-modal model. In particular, we show that our ensemble multi-modal model provably makes the same prediction for a multi-modal input when the number of added (or deleted or modified) basic elements to $\mathbf{m}_1, \mathbf{m}_2, \cdots, \mathbf{m}_T$ is no larger than $r_1, r_2, \cdots, r_T$. Intuitively, there is a considerable overlap between the space of randomly sub-sampled multi-modal inputs before the attack and those sub-sampled after the attack. This suggests that the alterations in the output prediction probabilities are constrained. Following~\cite{cohen2019certified,jia2020intrinsic}, the robustness guarantee is achieved by utilizing Neyman-Pearson Lemma~\cite{neyman1933ix}.

We conduct a systematic evaluation for our {\name} on two benchmark datasets for
multi-modal road segmentation and multi-modal emotion recognition tasks, respectively. We measure the performance lower bounds of our defense under adversarial attacks, with the constraint that the number of modified (or deleted or added) basic elements to each modality is bounded. We compare our {\name} with randomized ablation~\cite{levine2019sparse}, which is a state-of-the-art certified defense for unimodal models. Our experimental results show that our {\name} significantly outperforms randomized ablation when extending it to multi-modal models.

In summary, we make the following major contributions: 
\begin{itemize}
    \item We propose {\name}, the \emph{first} certified defense against adversarial attacks to multi-modal models.
    \item We derive the provable robustness guarantees of our {\name}.
    \item We conduct a systematic evaluation for our {\name} and compare it with state-of-the-art certified defense for unimodal models. 
\end{itemize}

\section{Background and Related Work}
\label{background}
Multi-modal models~\cite{kazakos2019epic,prakash2021multi,chumachenko2022self,fan2020sne} are designed to process information across multiple types of data, such as text, images, 3D point clouds, and audio, simultaneously. Multi-modal models
have shown impressive results across a variety
of applications, such as scene understanding~\cite{kazakos2019epic}, object detection~\cite{prakash2021multi, wagner2016multispectral, hassan2020learning}, sentiment analysis~\cite{chumachenko2022self, zadeh2016mosi,zadeh2018multimodal,kumar2020gated}, visual question answering~\cite{antol2015vqa, hu2020iterative}, and semantic segmentation~\cite{fan2020sne, liang2022multimodal}. 

For simplicity, we use $\mathbf{M}=(\mathbf{m}_1, \mathbf{m}_2, \cdots, \mathbf{m}_T)$ to denote a multi-modal input with $T$ modalities, where $\mathbf{m}_i$ represents the group of basic elements (pixels, images, audio) from the $i$th ($i=1,2,\cdots, T$) modality.

\subsection{Adversarial Attacks to Multi-modal Models} Many existing studies~\cite{cheng2023fusion,shen2020drift, zhang2022towards,tian2021can, xu2018fooling, shah2019cycle} showed that multi-modal models are vulnerable to adversarial attacks~\cite{goodfellow2014explaining}. For instance, Cheng et al.~\cite{cheng2023fusion} showed that the multi-modal auto-driving system can be undermined by a single-modal
attack that only aims at the camera modality, which is considered
less expensive to compromise.
Those attacks cause severe security and safety concerns for the deployment of multi-modal models in various real-world applications such as autonomous driving~\cite{fan2020sne, li2022deepfusion, wang2021pointaugmenting,teichmann2018multinet, natan2022end}. In our work, we consider a general attack, where an attacker could arbitrarily add (or delete or modify) a certain number of basic elements to each modality. For instance, when each basic element of a modality represents a pixel, an attacker could arbitrarily manipulate (e.g., modify) some pixel values for that modality. 

\subsection{Existing Defenses} 
Defenses against adversarial attacks can be categorized into \emph{empirical defenses} and \emph{certified defenses}. Empirical defenses~\cite{shafahi2019adversarial, wong2020fast, madry2017towards, kim2019single,yang2021defending, tian2021can,wang2021certified} cannot provide formal robustness guarantees under arbitrary attacks. Multiple works~\cite{athalye2018obfuscated,carlini2017adversarial,uesato2018adversarial} have shown that they can be bypassed by more advanced attacks. Existing certified defenses~\cite{cohen2019certified,jia2019certified,levine2019sparse,chiang2020certified,ye2020safer,lecuyer2019certified,liu2021pointguard,zeng2023certified,xiang2022patchcleanser,jia2022multiguard,pei2023textguard,zhang2023pointcert,yang2023graphguard} against adversarial attacks all focus on unimodal model whose input is only from a single modality. 
Among those defenses, randomized ablation~\cite{levine2019sparse,jia2021almost} achieves state-of-the-art certified robustness guarantee when an attacker could arbitrarily modify a certain number of basic elements to the input. Our experimental results show that randomized ablation achieves sub-optimal provable robustness guarantees when extended to multi-modal models. This is because when the attacker introduces perturbations with $l_0$ bounds across all modalities, the space of possible perturbed multi-modal inputs cannot be straightforwardly formulated as a $l_0$ ball.

We note that all the previously discussed certified defenses~\cite{cohen2019certified,jia2019certified,levine2019sparse,chiang2020certified,ye2020safer,lecuyer2019certified,zeng2023certified,xiang2022patchcleanser} are model-agnostic and scalable to large models. Another family of certified defenses~\cite{wong2018provable,gowal2018effectiveness,katz2017reluplex} proposed to derive the certified robustness guarantee of an unimodal model by conducting a layer-by-layer analysis. In general, those methods cannot be applied to general models and are not scalable to large neural networks.

\section{Problem Formulation}\label{problem}

We first introduce the threat model and then formally define certified defense against adversarial attacks to classification and segmentation tasks.
\subsection{Threat Model}
We discuss the threat model from the perspective of the attacker's goals, background knowledge, and capabilities. 

\myparatight{Attacker's goals} Given a multi-modal input and a multi-modal model, an attacker aims to adversarially perturb the multi-modal input such that the multi-modal model makes incorrect predictions for the perturbed multi-modal input.

\myparatight{Attacker's background knowledge and capabilities}
As we focus on the certified defense, we assume the
attacker has full knowledge about about the multi-modal model,
including its architecture and parameters.
We consider a strong attack to multi-modal models. In particular, given a multi-modal input, an attacker could simultaneously manipulate all modalities of the input~\cite{tian2021can, zhang2022towards}.
As a result, a multi-modal makes incorrect predictions for the perturbed multi-modal input. For example, to attack an auto-driving system, the attacker can add adversarial perturbation to both the depth image (captured by a LiDAR depth sensor) and the RGB image (captured by a camera) to lower the prediction quality. 

Formally, we denote a multi-modal input as $\mathbf{M}=(\mathbf{m}_1, \mathbf{m}_2, \cdots, \mathbf{m}_T)$, where $\mathbf{m}_i$ represents a group of elements of the $i$-th modality, the attacker could arbitrarily add (or delete or modify) at most $r_i$ elements to $\mathbf{m}_i$. For instance, when $\mathbf{m}_i$ represents an image, an attacker could arbitrarily change $r_i$ pixel values.

 We use $\mathbf{M}' = (\mathbf{m}'_1, \mathbf{m}'_2, \cdots, \mathbf{m}'_T)$ to denote the adversarial input. Without loss of generality, every modality can be rewritten as a list of it's basic elements. For example, an image (e.g., RGB image) can be written as a list of pixels, and an audio can be written as a list of audio frames. Therefore, we can denote $\mathbf{m}_i$ as a composition of basic elements denoted by $[m^1_i, m^2_i,\cdots,m^{n_i}_{i}]$, where $m^j_i$ represents the $j$-th basic element in the $i$-th modality, and $n_i$ represents the total number of basic elements in the $i$-th modality. We denote the number of basic elements in each modality after the attack as $n_1',n_2',\ldots, n_T'$, respectively. For the image modality, we know the number of basic elements (pixels) is fixed. However, for some other modalities like audio, the attacker is able to change the number of basic elements (e.g., audio frames) via addition or deletion. 

Hence, we define three kinds of attacks for each modality: modification attack, addition attack, and deletion attack.
We use $\mathcal{S}(\mathbf{m}_i,r_i)$ to denote the set of all possible $\mathbf{m}'_i$ when an attacker could add (or delete or modify) at most $r_i$ basic elements in $\mathbf{m}_i$. For simplicity, we use $\mathbf{R} = (r_1,r_2,\ldots,r_T)$ to denote the added (or deleted or modified) basic elements to all modalities. Then we use $\mathcal{S}(\mathbf{M},\mathbf{R}) = \mathcal{S}(\mathbf{m}_1,r_1) \times \mathcal{S}(\mathbf{m}_2,r_2)\ldots \times \mathcal{S}(\mathbf{m}_T,r_T)$ to denote the set of all possible adversarial inputs $\mathbf{M}' = (\mathbf{m}'_1, \mathbf{m}'_2, \cdots, \mathbf{m}'_T)$.

\subsection{Certifiably Robust Multi-modal Prediction}
For classification tasks, suppose we have a multi-modal classifier $G$. Given a test sample $(\mathbf{M},y)$, where $y$ is the ground truth label, we say $G$ is \emph{certifiably stable} for $\mathbf{M}$ if the predicted label remains unchanged under attack:
\begin{align}
    &G(\mathbf{M}) = G(\mathbf{M}')
    , \forall \mathbf{M}' \in \mathcal{S}(\mathbf{M},\mathbf{R}).\end{align}
If this unchanged label is the ground-truth label of $\mathbf{M}$, i.e., $G(\mathbf{M}) = y$, then we say the classifier $G$ is  \emph{certifiably robust} for this test sample.

For segmentation tasks, without loss of generality, we assume the multi-modal model outputs the segmentation result for one of the input modalities (denoted by $\mathbf{m}_o$) with $n_o$ basic elements (e.g., pixels). Then the output contains $n_o$ labels. For example, if RGB image is one of the input modalities, the output can be a segmentation of this RGB image, which contains a label for each pixel in the RGB image. Unless otherwise mentioned, we assume that the attacker performs modification attacks on $\mathbf{m}_o$ (please refer to Appendix~\ref{appendix-segmentation} for deletion and addition attacks on $\mathbf{m}_o$).

We can think of the multi-modal segmentation model $G$ as composed of multiple classifiers denoted by $G_1,G_2,\ldots, G_{n_o}$. Each classifier $G_j$ predicts a label $G_j(\mathbf{M})$ for $m^j_o$ (the $j$-th basic element of $\mathbf{m}_o$). The ground truth $y$ also includes $n_o$ labels, denoted by $y_1,y_2,\ldots, y_{n_o}$. We use $G_j(\mathbf{M}')$ to denote the predicted label for $m^j_o$ after the attack. We say $G_{j}$ is \emph{certifiably stable} for a basic element (e.g., a pixel) $m^j_o$ if:
\begin{align}
    &G_j(\mathbf{M}) = G_j(\mathbf{M}')
    , \forall \mathbf{M}' \in \mathcal{S}(\mathbf{M}, \mathbf{R}),
\end{align}
which means the predicted label for the the $j$-th basic element of $\mathbf{m}_o$ remains unchanged under attack. If $G_j(\mathbf{M}) = y_j$, then we term $G_{j}$ as \emph{certifiably robust} for $m^j_o$. 

By deriving a lower bound on the number of basic elements whose predictions are certifiably robust, we can guarantee the segmentation quality for a test sample, measured via metrics such as Certified Pixel Accuracy, Certified F-score, or Certified IoU.

\section{Our Design}
\subsection{Independent Sub-sampling}
In this section, we will first outline a universal sub-sampling method~\cite{levine2019sparse,jia2021almost,liu2021pointguard}, and then demonstrate its application across various multi-modal tasks.

\myparatight{Sub-sampling Strategy} We repeatedly randomly sub-sample $k_i$ basic elements (e.g., pixels) from the $i$-th modality $\mathbf{m}_i = [m^1_i,m^2_i, \cdots,m^{n_i}_i]$ without replacement. For simplicity, we use $\mathcal{Z}=(\mathbf{z}_1,\mathbf{z}_2, \ldots, \mathbf{z}_T)$ to denote the randomly sampled multi-modal input. 
Thus, we have $|\mathbf{z}_i| = k_i$ for all $i = 1,2,\ldots, T$. This sampling strategy exhibits versatility by being applicable across various modalities and tasks. It can be applied for classification tasks, e.g., emotion recognition. And it can also be
employed for segmentation tasks, e.g., road segmentation. Figure~\ref{fig-subsampling} in Appendix provides a visualization of this sub-sampling method.

Next, we first apply this sampling strategy to build an ensemble classifier for classification tasks.
\subsection{Certify Multi-modal Classification}

\myparatight{Ensemble Classifier} Given a testing input $\mathbf{M} = (\mathbf{m}_1,\mathbf{m}_2, \ldots, \mathbf{m}_T)$, we use $\mathcal{Z}= (\mathbf{z}_1,\mathbf{z}_2, \ldots, \mathbf{z}_T)$ to denote the randomly sub-sampled multi-modal input. We denote the multi-modal model by $g$. For simplicity, we use $g(\mathcal{Z})$ and $y$ to denote the predicted label and the true label. As $\mathcal{Z}$ is randomly sub-sampled, $g(\mathcal{Z})$ is also random. Given an arbitrary label $l \in \{1,2,\cdots, C\}$  ($C$ is the total number of classes), we use $p_l$ to denote the probability that the predicted label $g(\mathcal{Z})$ is $l$. Formally, we have $p_l = \text{Pr}(l = g(\mathcal{Z}))$. We call $p_l$ \emph{label probability}. In practice, it is computationally expensive to calculate the exact label probabilities. Following~\cite{cohen2019certified,jia2019certified,levine2019sparse}, we use Monte Carlo sampling to estimate a lower 
 bound or upper bound of $p_l$, denoted as $\underline{p_l}$ and $\overline{p_l}$ respectively. This is achieved by randomly sample $N$ ablated inputs from the distribution $\mathcal{Z}$, represented as $\mathbf{Z}_1,\mathbf{Z}_2,\cdots, \mathbf{Z}_N$, and then count the label frequency $N_l = \sum_{i=1}^{N}\mathbb{I}(g(\mathbf{Z}_i) = l)$ for each label $l$. Our ensemble classifier $G$ then predicts the label with the largest frequency $N_l$. For simplicity, we denote this label by $A$ and use $\underline{p_A}$ to represent $A$'s label probability lower bound. We define the runner-up label $B$ as the label with the second highest label frequency, i.e., $B=\text{argmax}_{l\neq A}{N_l}$. We present our certification result below:
 
\begin{theorem}[Certification for classification]
\label{theorem_of_certified_radius_classification}
Suppose we have a multi-modal test input $\mathbf{M}$ and a base multi-modal classifier $g$. Our ensemble classifier $G$ is as defined as above. We denote $A = G(\mathbf{M})$ and use $\underline{p_A}$ to denote the label probability lower bound for the label $A$. We use $B$ to denote the runner-up class and use $\overline{p_B}$ to denote the label probability upper bound for the label $B$. We define $\delta_{l}=\underline{p_A} - \frac{\lfloor\underline{p_A}\prod_{i=1}^T{n_i \choose k_i}\rfloor}{\prod_{i=1}^T{n_i \choose k_i}} $ and $\delta_{u}=\frac{\lceil\overline{p_B}\prod_{i=1}^T{n_i \choose k_i}\rceil}{\prod_{i=1}^T{n_i \choose k_i}} -\overline{p_B} $. 
Given a perturbation size $\mathbf{R} = (r_1,r_2,\ldots, r_T)$, we have the following:
\begin{align}
    &G(\mathbf{M}) = G(\mathbf{M}'),
    \forall \mathbf{M}' \in \mathcal{S}(\mathbf{M}, \mathbf{R})
\end{align}
if:
\begin{align}
&\frac{\prod_{i=1}^T{n_i \choose k_i}}{\prod_{i=1}^T{n'_i \choose k_i}}(\underline{p_A}-\delta_{l}-1+ \frac{\prod_{i=1}^T{e_i \choose k_i}}{\prod_{i=1}^T{n_i \choose k_i}}) \\
\geq& 
 \frac{\prod_{i=1}^T{n_i \choose k_i}}{\prod_{i=1}^{T}{n'_i \choose k_i}}(\overline{p_B}+\delta_{u})+1- \frac{\prod_{i=1}^T{e_i \choose k_i}}{\prod_{i=1}^T{n'_i \choose k_i}}
\end{align}
where $e_i=n_i-r_i$ and $n'_i=n_i$ for modification attack; $e_i=n_i$ and $n'_i=n_i+r_i$ for addition attack; $e_i=n_i-r_i$ and $n'_i=n_i-r_i$ for deletion attack, where $i = 1,2, \ldots, T$ is the modality index.
\end{theorem}

\begin{proof}
Please refer to Appendix~\ref{proof_of_theorem_1}. 
\end{proof}

\myparatight{Computing $\underline{p_B}$ and $\overline{p_A}$} Following~\cite{cohen2019certified,jia2019certified,jia2020intrinsic}, we apply Monte Carlo sampling
to approximate $\underline{p_B}$ and $\overline{p_A}$. We first randomly sub-sample $N$ multi-modal inputs from the test input $\mathbf{M}$, and we denote these ablated inputs as $\mathbf{Z}_1,\mathbf{Z}_2,\cdots, \mathbf{Z}_N$. We denote the number of sub-sampled inputs that predicts for the label $l$ as $N_l$, i.e., $N_l = \sum_{i=1}^{N}\mathbb{I}(g(\mathbf{Z}_i) = l)$. Then, the frequency $N_l$ of any label $l$ follows a binomial
distribution. Therefore, we can apply
Clopper-Pearson~\cite{clopper1934use} based method
to estimate $\underline{p_B}$ and $\overline{p_A}$ with predefined confidence level $1-\alpha$:
\begin{align}
&\underline{p_A} = Beta(\frac{\alpha}{C}; N_A, N-N_A+1)\\
& \overline{p_B} = Beta(1-\frac{\alpha}{C}; N_B, N-N_B+1),
\end{align} where $A$ represents the predicted label, i.e., $A=\text{argmax}_{l }{N_l}$, and $B$ represents the runner-up label, i.e., $\text{argmax}_{l\neq A}{N_l}$. $Beta(\beta; \lambda, \theta)$ calculates the $\beta$-th quantile of the Beta distribution
given shape parameters $\lambda$ and $\theta$. We divide $\alpha$ by the number of classes because we estimate bounds for $C$ classes simultaneously~\cite{jia2020intrinsic}. By Bonferroni
correction, if we use $1-\alpha/C$ as the confidence level to estimate each bound, then the overall confidence level for the $C$ classes is at least $1-\alpha$.

\subsection{Certify Multi-modal Segmentation}\label{sec-certify-segmentation}
In this section, we extend our certification method for classification tasks to certify multi-model segmentation tasks. Segmentation tasks are essentially a variant of classification since each basic element (e.g., a pixel) in one of the input modalities (e.g., an image) is assigned a label. We denote this input modality as $\mathbf{m}_o$, and denote the $j$-th basic element of $\mathbf{m}_o$ as $m^j_o$. Suppose $\mathbf{m}_o$ has $n_o$ basic elements. If we naively apply union bound, certifying the test input with overall confidence level $1-\alpha$ requires certifying each basic element with confidence level $1-\frac{\alpha}{n_o}$, which becomes hard when $n_o$ grows large. To maximize the number of certified basic elements, Fischer et al.~\citep{fischer2021scalable} utilized the Holm–Bonferroni method~\cite{holm1979simple}, originally designed for \emph{Multiple Hypothesis Testing}. Specifically, this method tends to certify basic elements with confident predictions, while abstaining ambiguous basic elements. Furthermore, this method guarantees that the probability of mistakenly reporting at least one non-certifiably-stable basic element as certified is limited at $\alpha$. 
In this work, we adapt the approach from Fischer et al.~\citep{fischer2021scalable} to multi-modal scenarios.

\myparatight{Ensemble Classifiers for Segmentation}
Given a testing input $\mathbf{M}$, we use $\mathcal{Z} = (\mathbf{z}_1,\mathbf{z}_2,\ldots,\mathbf{z}_T)$ to denote the randomly sub-sampled input. The base multi-modal segmentation model can be seen as a composition of multi-modal classifiers $g_1, g_2, \ldots, g_{n_o}$, where $g_j$ predicts a label for the basic element $m^j_o$. We use $g_j(\mathcal{Z})$ to denote the predicted label for $m^j_o$. We randomly sample $N$ ablated inputs from the distribution $\mathcal{Z}$, and represent them as $\mathbf{Z}_1,\mathbf{Z}_2,\cdots, \mathbf{Z}_N$. For each basic element $m^j_o$ and each label $l$, we count the label frequency $N^j_l = \sum_{i=1}^{N}\mathbb{I}(g_j(\mathbf{Z}_i) = l)$. The ensemble classifier for $m^j_o$ (denoted by $G_j$) then predicts the the label $l$ with the highest label frequency $N^j_l$, i.e., $G_j(\mathbf{M}) = \text{argmax}_l N^j_l$. We say $G_j$ is \emph{certifiably stable} for  $m^j_o$ if the predicted label of $G_j$ for $m^j_o$ remains unchanged under attack, i.e., $
    G_j(\mathbf{M}) = G_j(\mathbf{M}'), \forall \mathbf{M}' \in \mathcal{S}(\mathbf{M}, \mathbf{R})$. Next, we discuss how to certify as many basic elements as possible given that the possibility of mistakenly certifying a non-certifiably-stable basic element is at most $\alpha$.
 
\myparatight{Calculate a Confidence Level for Each Basic Element} For each basic element $m^j_o$, we denote the number of ablated inputs that predicts the label $l$ for this component as $N^j_l$. We denote the total number of ablated inputs as $N$. We define:
\begin{align}\label{eqn-probability-lower}
&\underline{p_A}(\alpha_j) = Beta(\frac{\alpha_j}{C}; N^j_A, N-N^j_A+1)\\
& \overline{p_B}(\alpha_j) = Beta(1-\frac{\alpha_j}{C}; N^j_B, N-N^j_B+1),
\end{align} where $A$ represents the predicted label for this basic element, i.e., $\text{argmax}_{l}{N^j_l}$, and $B$ represents the runner-up label for this component, i.e., $\text{argmax}_{l\neq A}{N^j_l}$. Then we define:
\begin{align}
&\alpha_j^* = \min_{\alpha_j} \\
s.t.,\text{ }
&\frac{\prod_{i=1}^T{n_i \choose k_i}}{\prod_{i=1}^T{n'_i \choose k_i}}(\underline{p_A}(\alpha_j)-\delta_{l}-1+ \frac{\prod_{i=1}^T{e_i \choose k_i}}{\prod_{i=1}^T{n_i \choose k_i}}) \\
\geq& 
 \frac{\prod_{i=1}^T{n_i \choose k_i}}{\prod_{i=1}^{T}{n'_i \choose k_i}}(\overline{p_B}(\alpha_j)+\delta_{u})+1- \frac{\prod_{i=1}^T{e_i \choose k_i}}{\prod_{i=1}^T{n'_i \choose k_i}},
\end{align}
where $n_i$, $n_i'$, $e_i$, $k_i$,  $\delta_{l}$ and $\delta_{u}$ are defined as in Theorem~\ref{theorem_of_certified_radius_classification}. Then with probability at least $1-\alpha^*_j$, the basic element $m^j_o$ is certifiably stable (the output label of this basic element cannot be changed by the attacker) according to Theorem~\ref{theorem_of_certified_radius_classification}. In practice, we calculate $\alpha^*_j$ by binary search. If such an $\alpha^*_j$ does not exist, the binary search algorithm returns $1$ instead.

\myparatight{Apply Holm-Bonferroni method}Using the computed values of $\alpha_j^*$, we employ the Holm-Bonferroni method~\cite{holm1979simple} to determine the basic elements eligible for certification. This method maximizes the number of certified basic elements, while at the same time ensures that the possibility of mistakenly certifying a non-certifiably-stable basic element remains within the limit of $\alpha$. Specifically, we have two steps:
\begin{itemize}
    \item Step 1: We order $\alpha_j^*$-values ($j=1,2,\cdots, n_o$) in ascending order so that we have $\alpha_{(1)}^*\leq\alpha_{(2)}^*\leq\cdots\leq\alpha_{(n_o)}^*$.
    \item Step 2: Calculate $L = \min\{j: \alpha_{(j)}^* > \frac{\alpha}{n_o+1-j}\}$.
\end{itemize}
We report all basic elements $m^j_o$ for which $\alpha_j^* < \alpha_{(L)}^*$ as certifiably stable (the output labels of these basic elements cannot be changed by the attacker), and the predictions for other basic elements are abstained. In Section~\ref{section-evaluation}, we derive certified metrics, e.g., Certified Pixel Accuracy, from these certifiably stable basic elements.

\begin{figure*}
\centering

\subfloat[$r_2 = r_1$]{\includegraphics[width=0.24\textwidth]{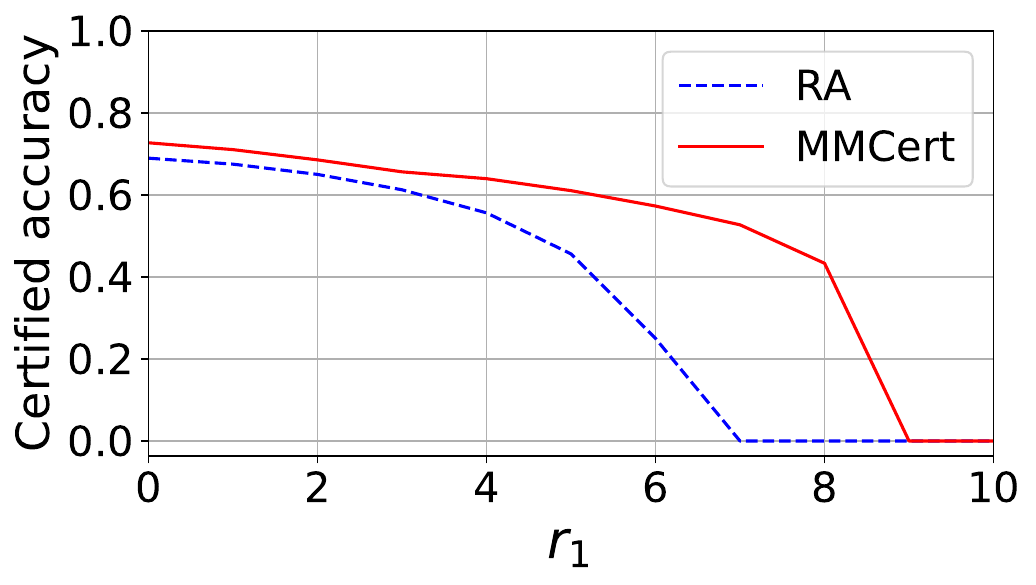}}\vspace{1mm}
\subfloat[$r_2 = 2r_1$]{\includegraphics[width=0.24\textwidth]{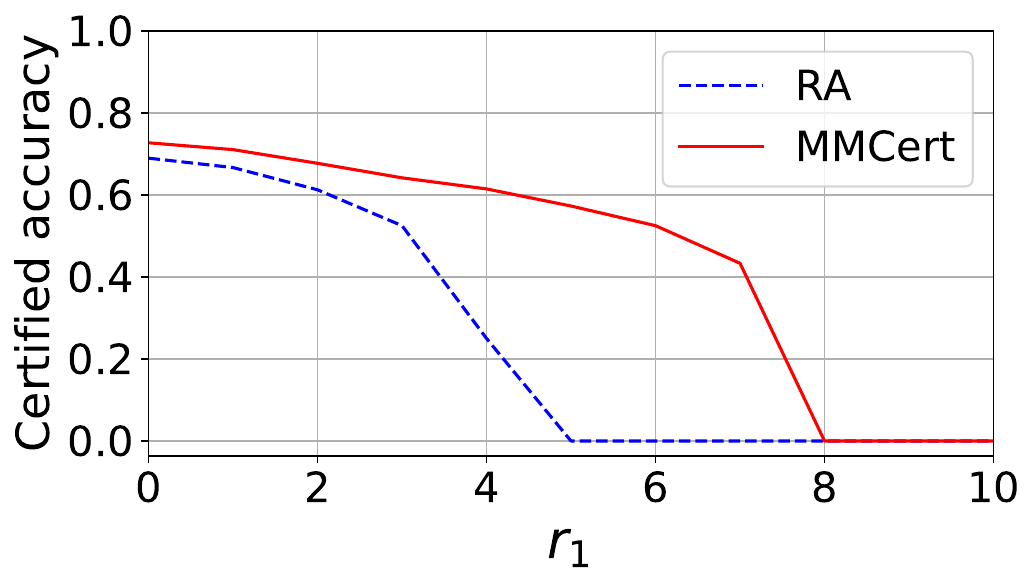}}\vspace{1mm}
\subfloat[$r_2 = 3r_1$]{\includegraphics[width=0.24\textwidth]{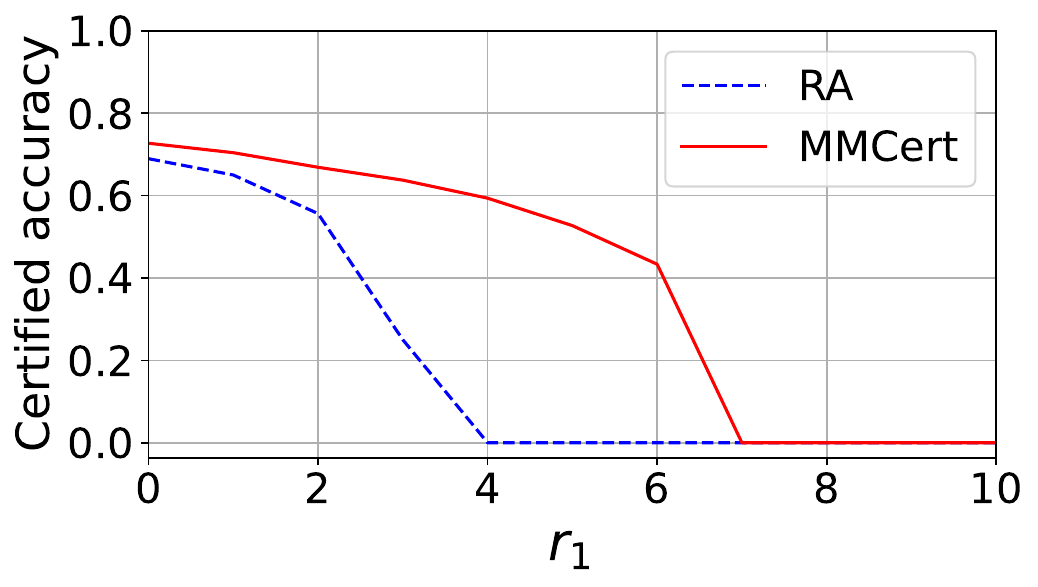}}\vspace{1mm}
\subfloat[$r_2 = 4r_1$]{\includegraphics[width=0.24\textwidth]{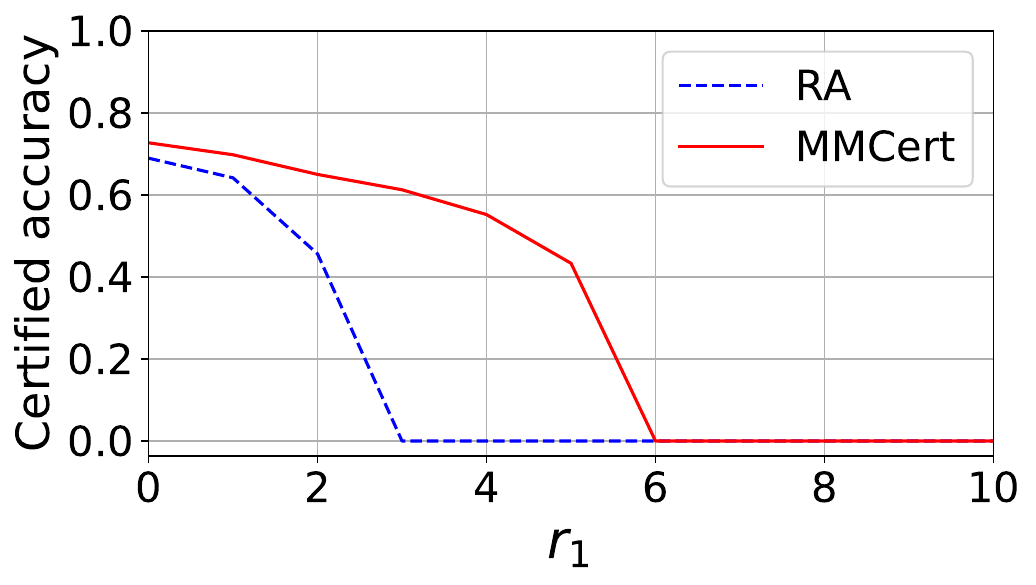}}\vspace{1mm}
\vspace{-4mm}
\caption{Compare our {\name} with randomized ablation on RAVDESS Dataset.
}
\label{exp-ravdess}
\end{figure*}

\begin{figure*}
\vspace{5mm}
\centering
{\includegraphics[width=0.24\textwidth]{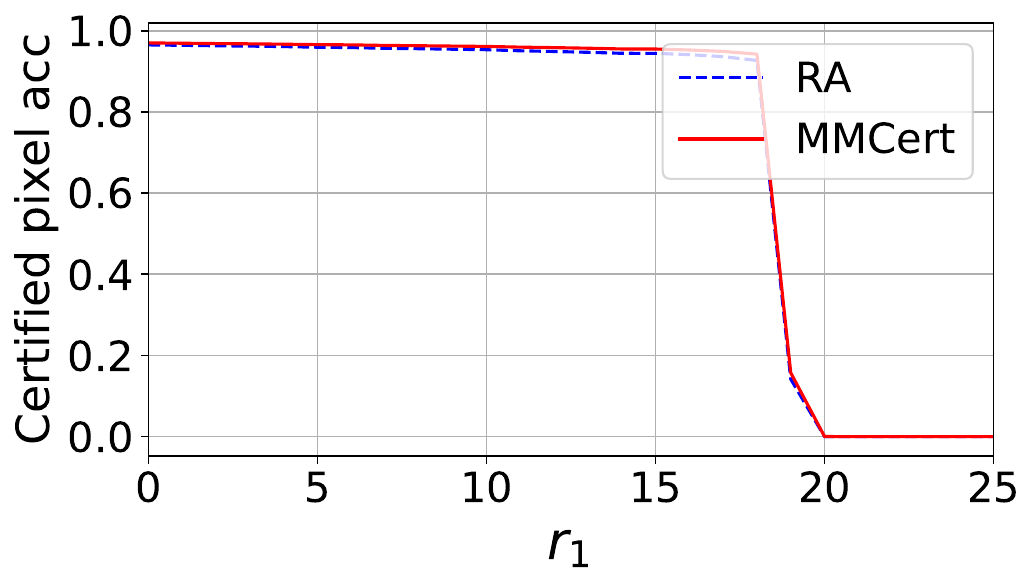}}
{\includegraphics[width=0.24\textwidth]{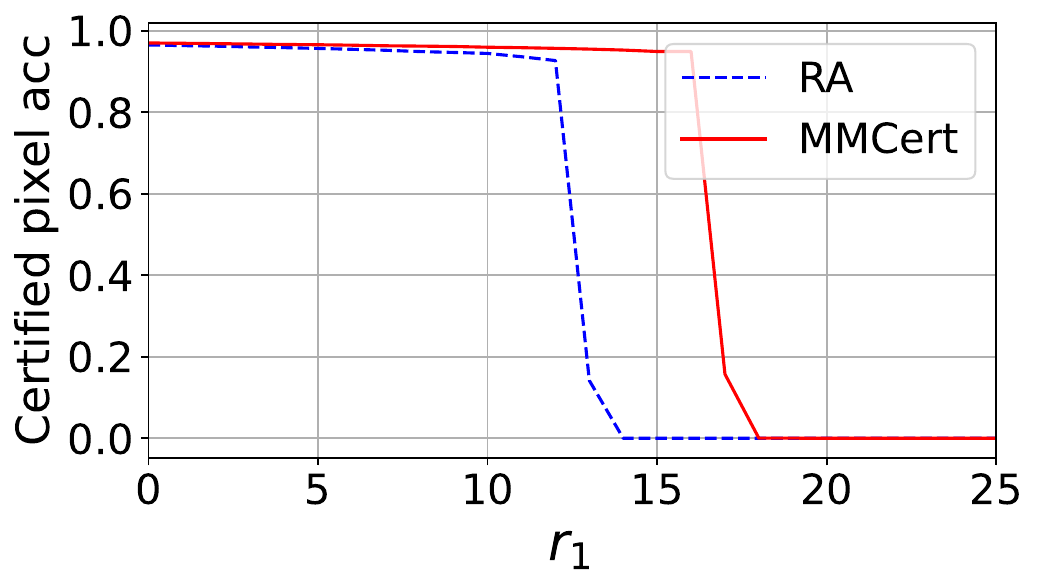}}
{\includegraphics[width=0.24\textwidth]{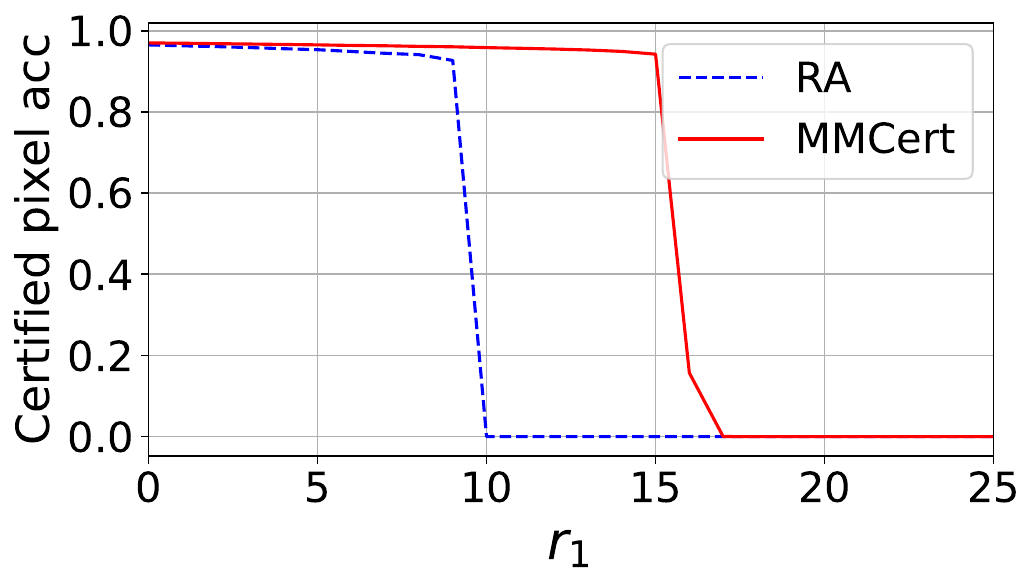}}
{\includegraphics[width=0.24\textwidth]{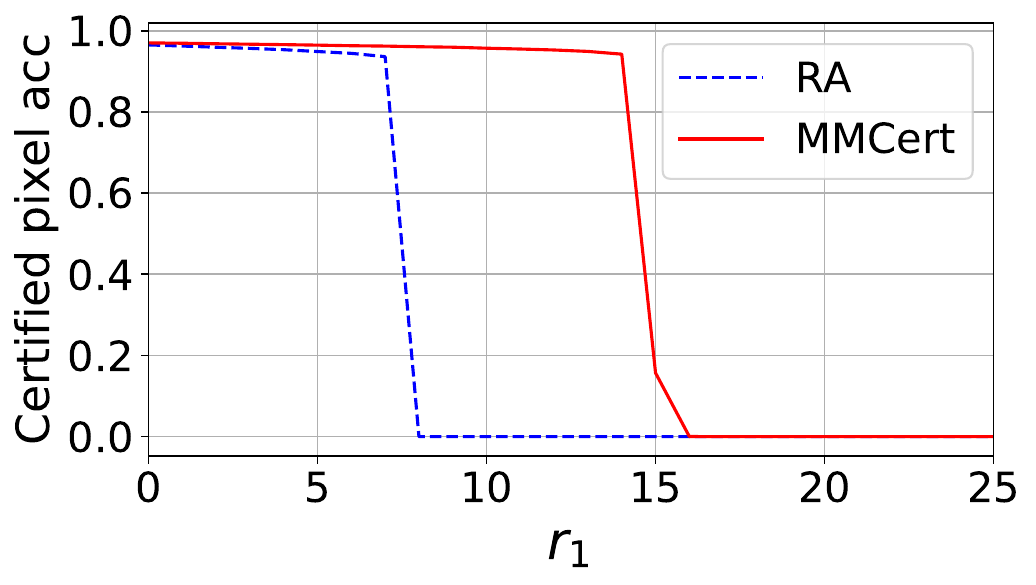}}
{\includegraphics[width=0.24\textwidth]{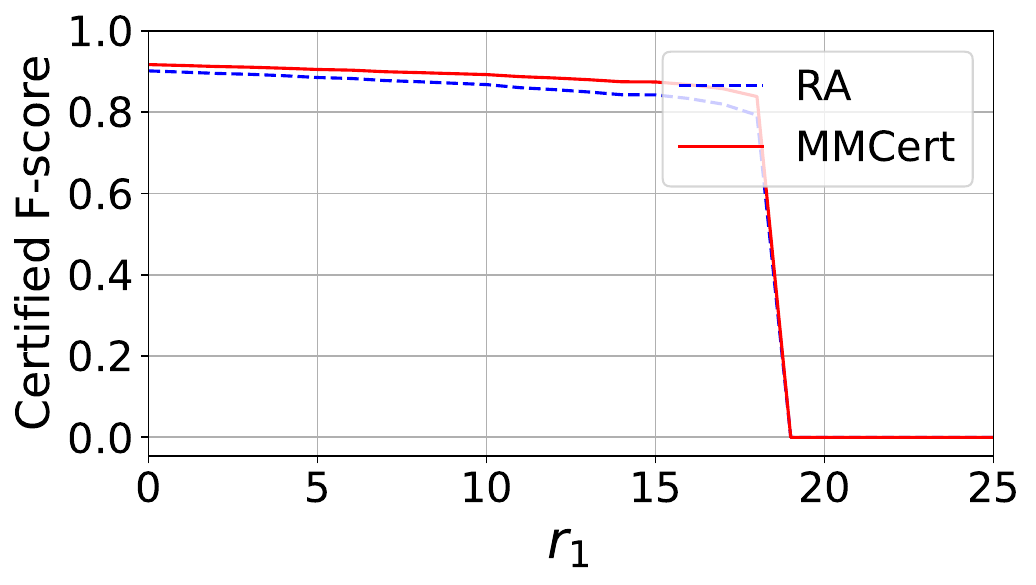}}
{\includegraphics[width=0.24\textwidth]{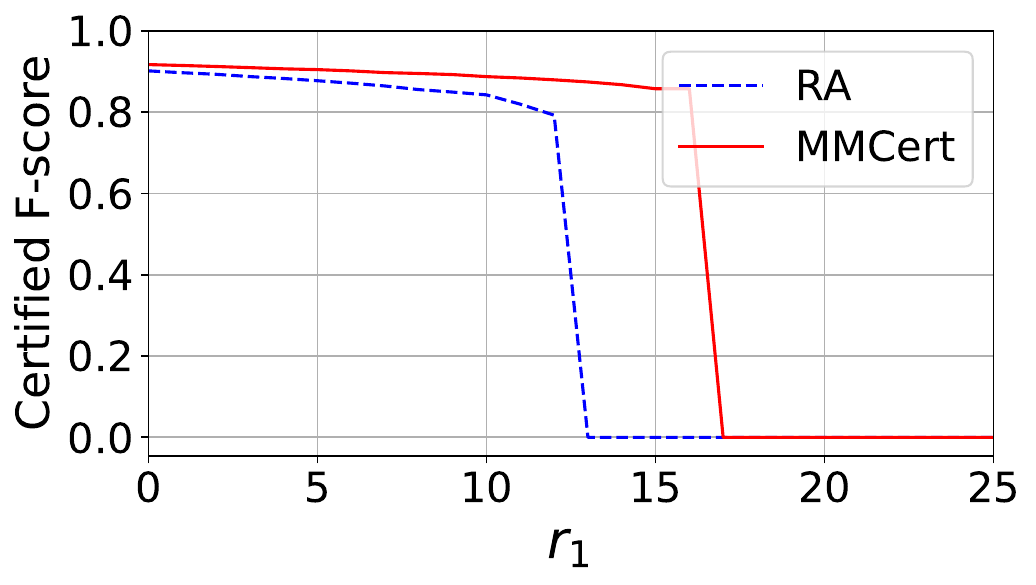}}
{\includegraphics[width=0.24\textwidth]{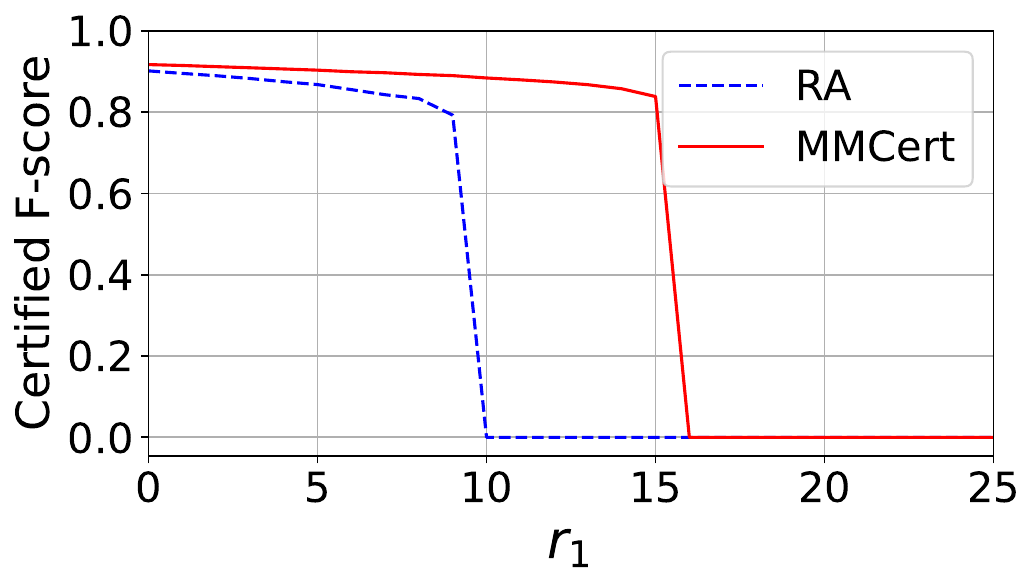}}
{\includegraphics[width=0.24\textwidth]{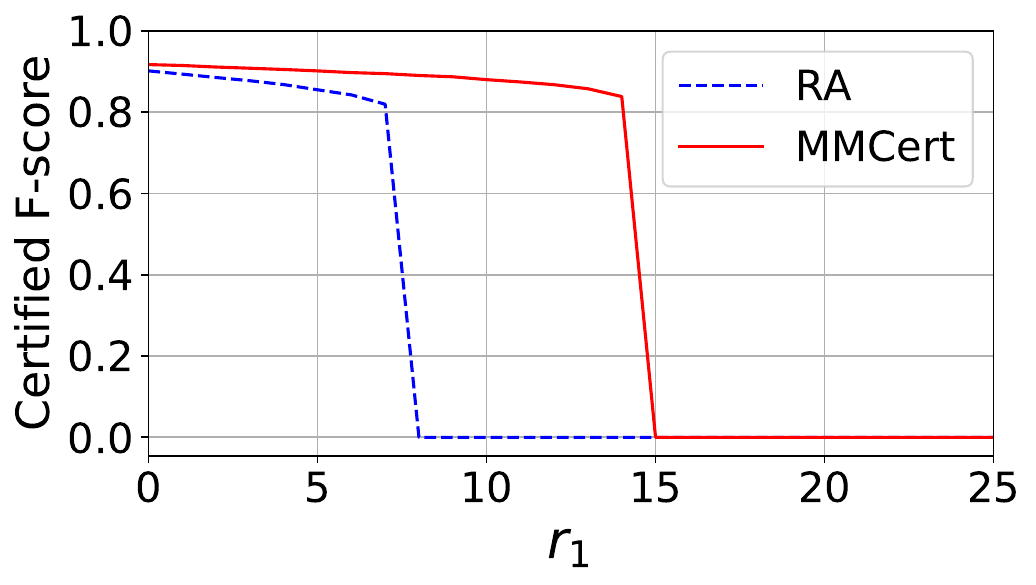}}

\subfloat[$r_2 = r_1$]{\includegraphics[width=0.24\textwidth]{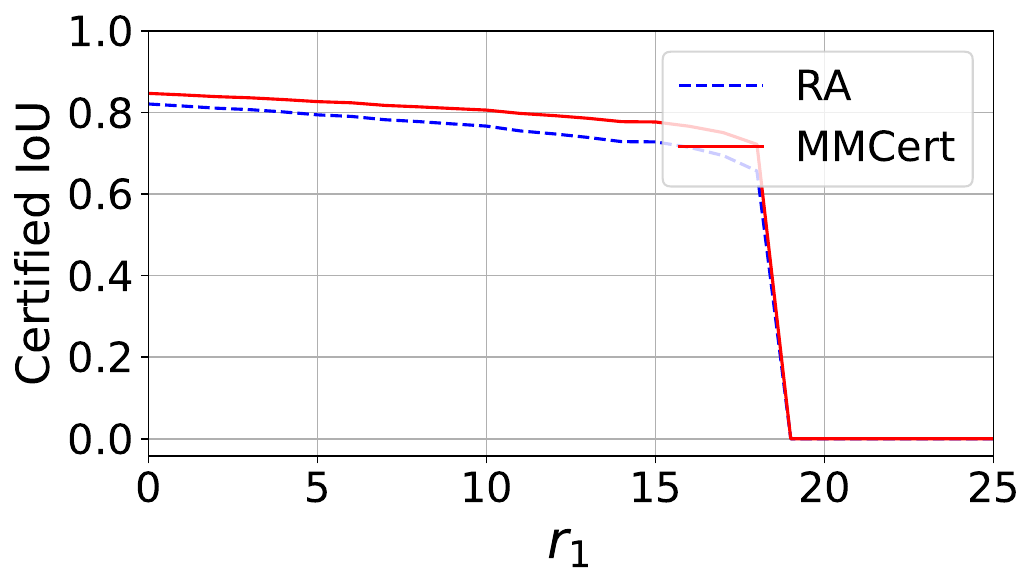}}\vspace{1mm}
\subfloat[$r_2 = 2r_1$]{\includegraphics[width=0.24\textwidth]{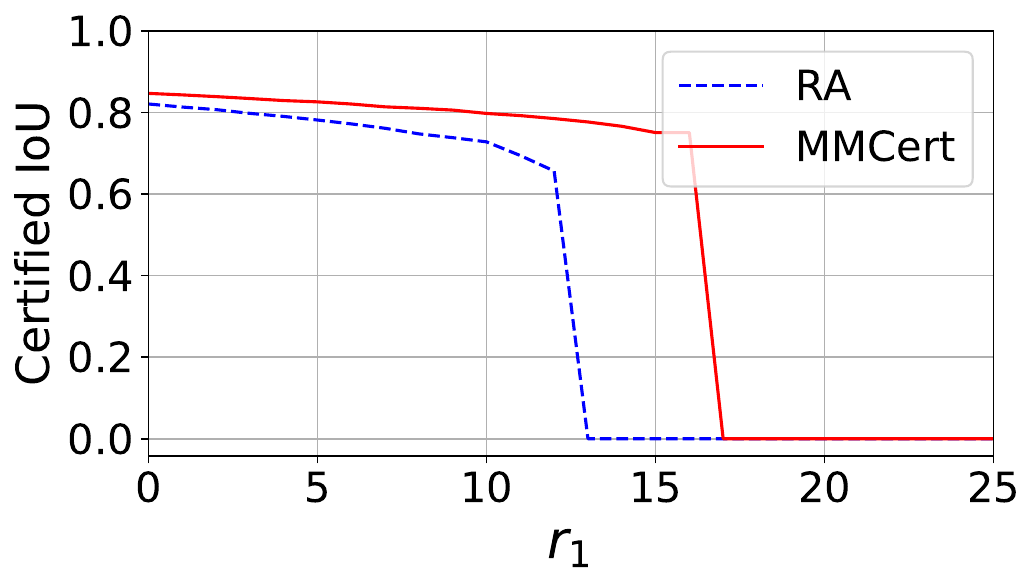}}\vspace{1mm}
\subfloat[$r_2 = 3r_1$]{\includegraphics[width=0.24\textwidth]{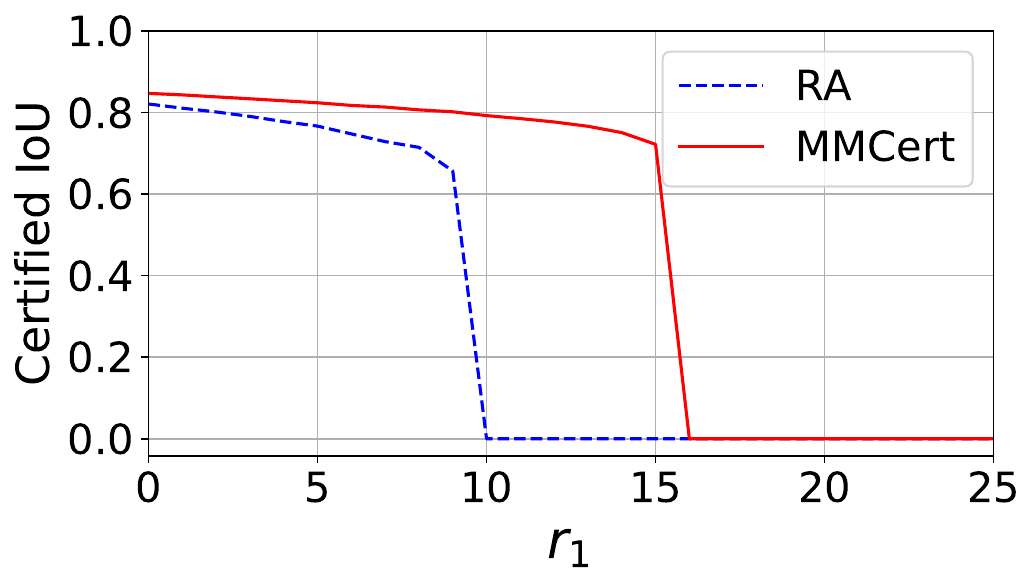}}\vspace{1mm}
\subfloat[$r_2 = 4r_1$]{\includegraphics[width=0.24\textwidth]{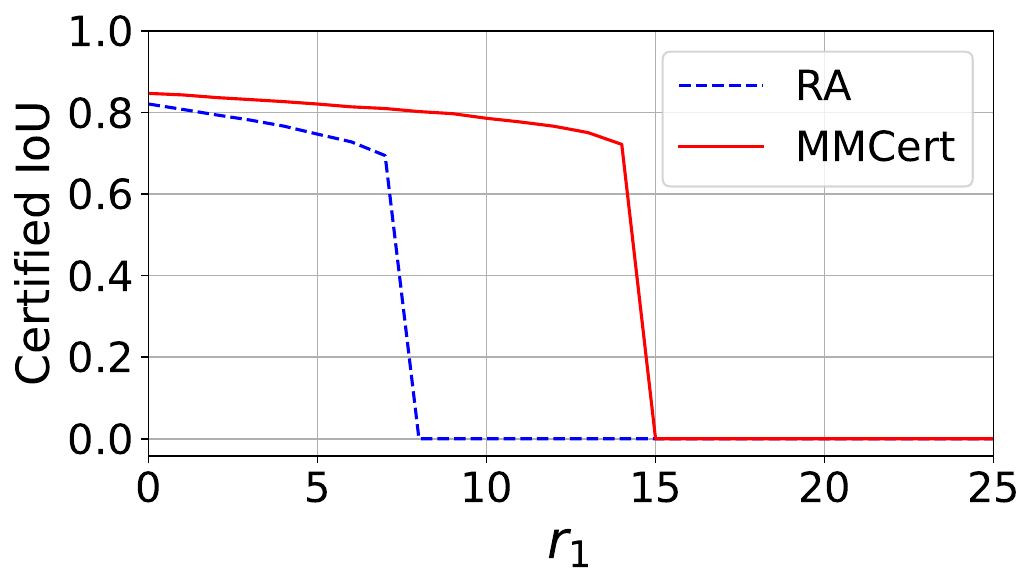}}\vspace{1mm}
\vspace{-4mm}
\caption{Compare our {\name} with randomized ablation on KITTI Road Dataset. Certified Pixel Accuracy (first row), Certified F-score (second row) and Certified IoU (third row) are considered.
}
\label{exp-kitti}
\vspace{-7mm}
\end{figure*}
\begin{figure*}
\centering
\subfloat[$r_2 = r_1$]{\includegraphics[width=0.24\textwidth]{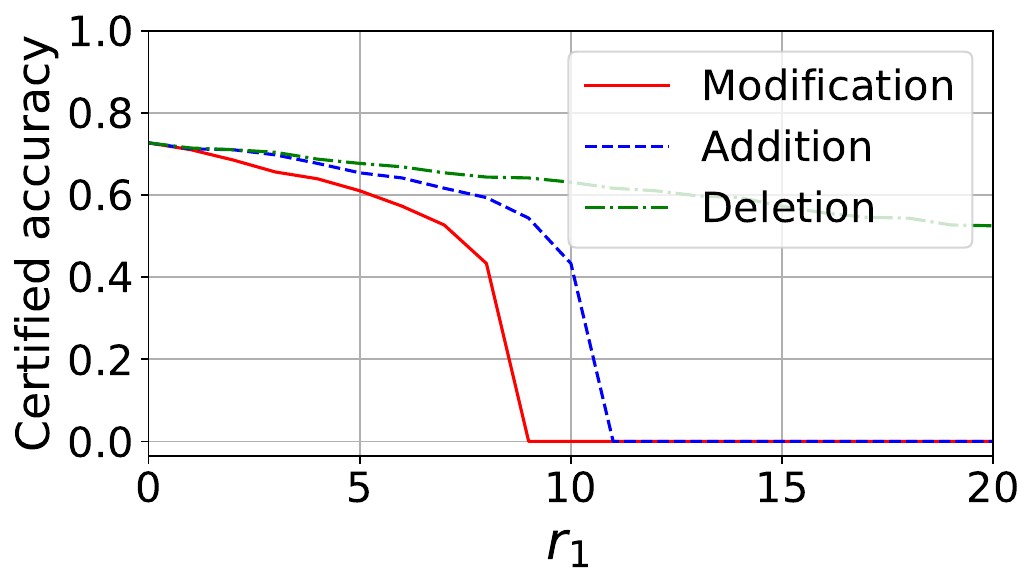}}\vspace{1mm}
\subfloat[$r_2 = 2r_1$]{\includegraphics[width=0.24\textwidth]{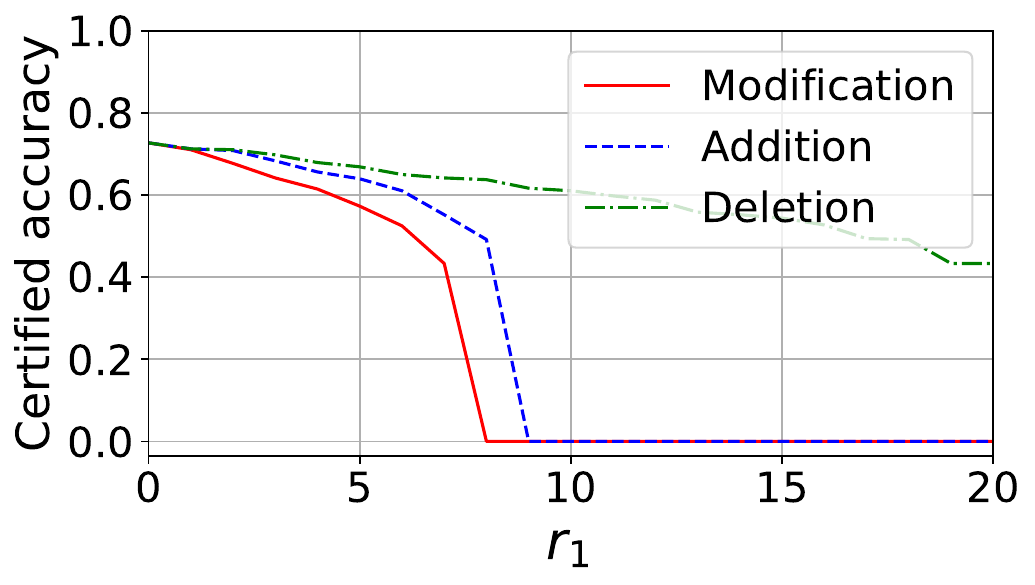}}\vspace{1mm}
\subfloat[$r_2 = 3r_1$]{\includegraphics[width=0.24\textwidth]{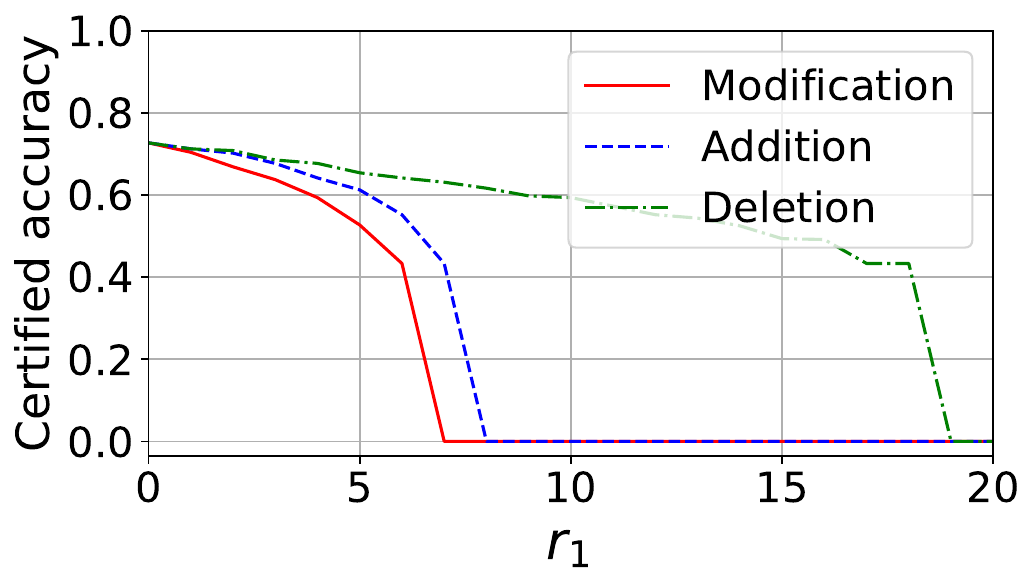}}\vspace{1mm}
\subfloat[$r_2 = 4r_1$]{\includegraphics[width=0.24\textwidth]{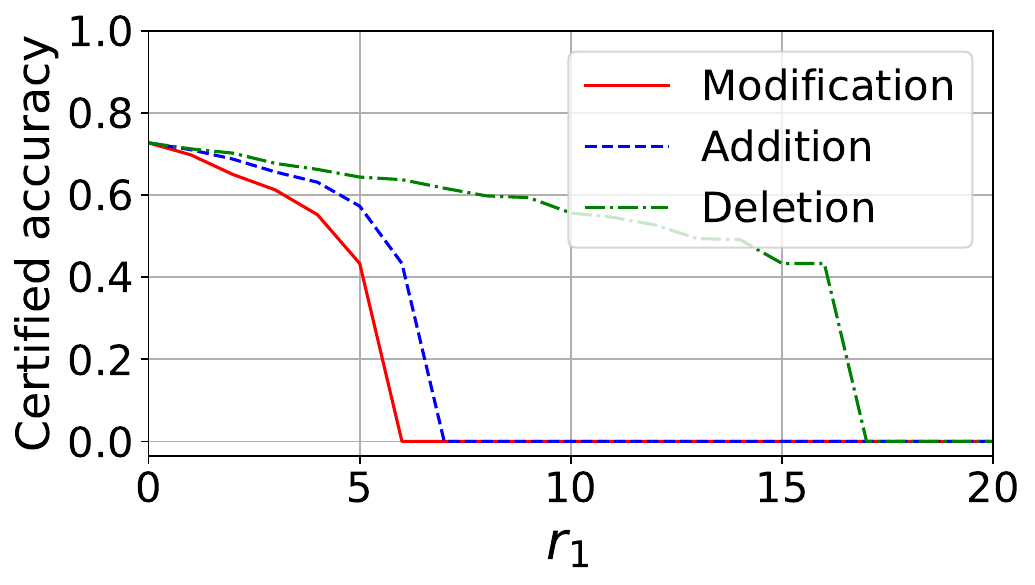}}\vspace{1mm}
\vspace{-4mm}
\caption{Compare different attack types on RAVDESS Dataset.
}
\label{exp-different_attacks}
\end{figure*}
\begin{figure*}
\vspace{0mm}
\centering
{\includegraphics[width=0.24\textwidth]{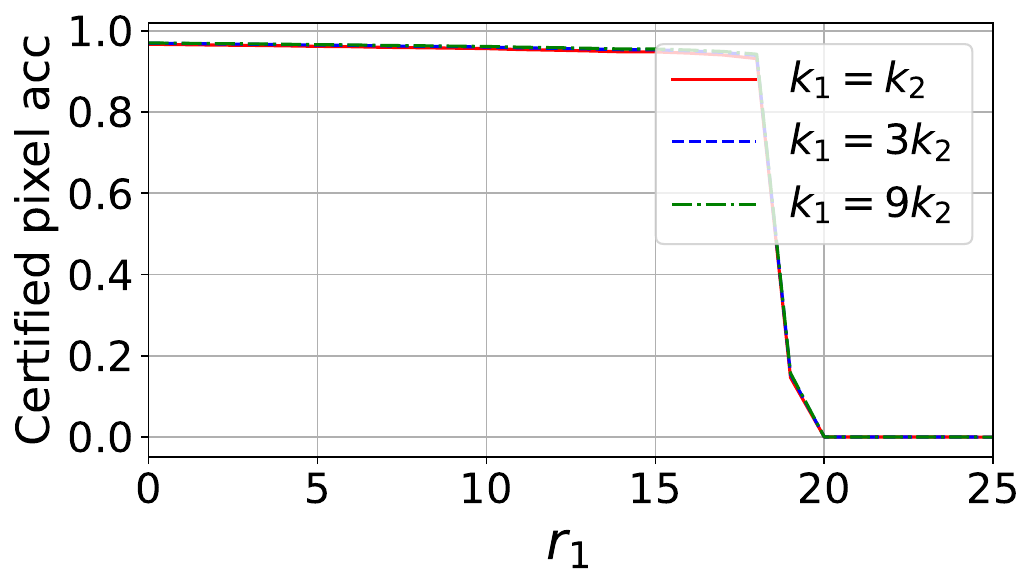}}
{\includegraphics[width=0.24\textwidth]{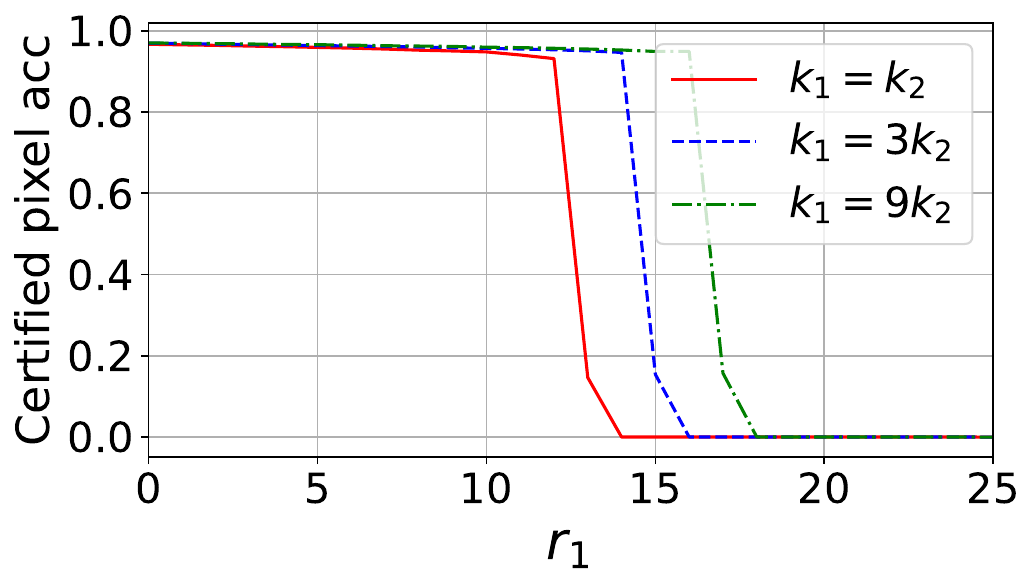}}
{\includegraphics[width=0.24\textwidth]{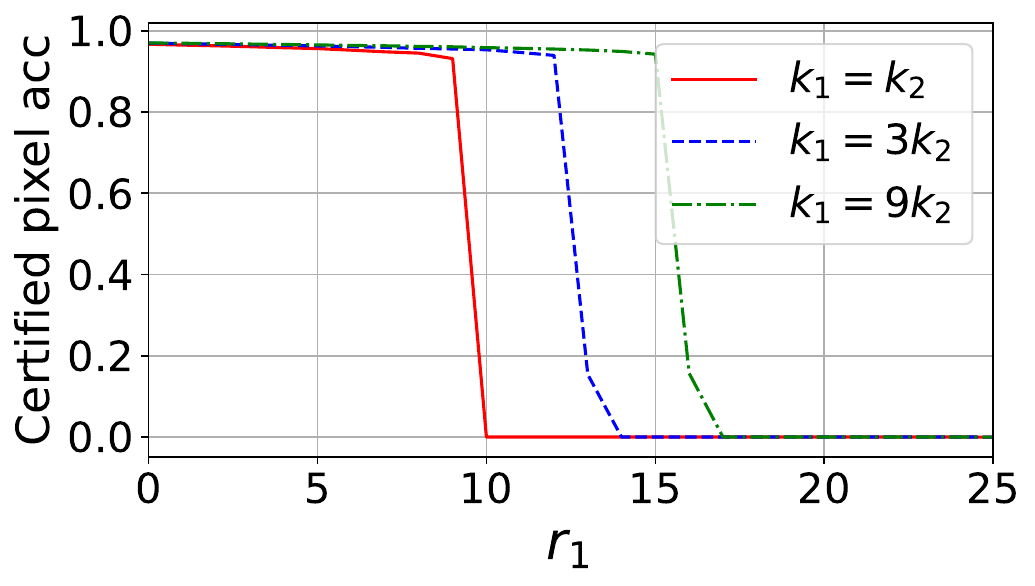}}
{\includegraphics[width=0.24\textwidth]{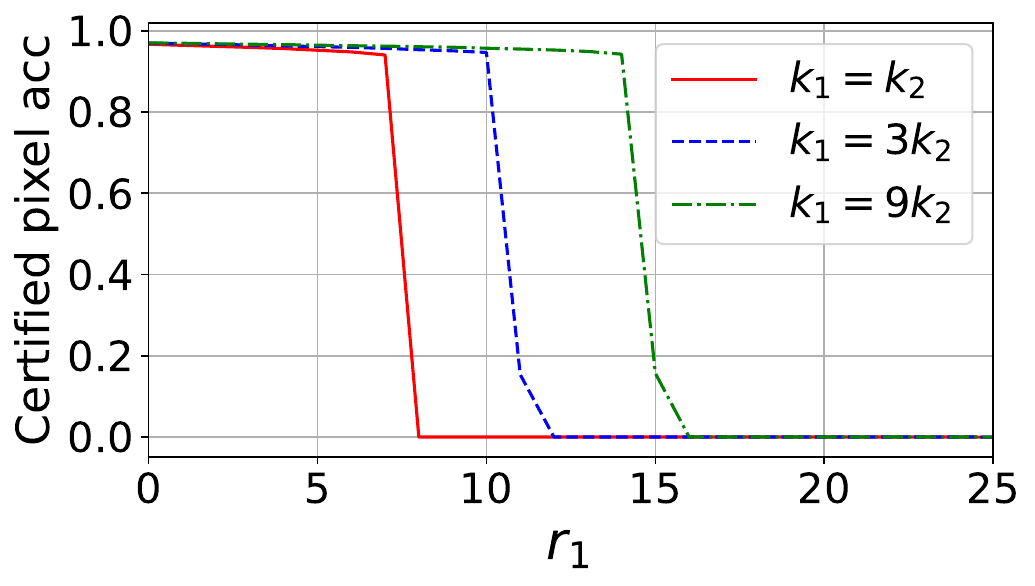}}
{\includegraphics[width=0.24\textwidth]{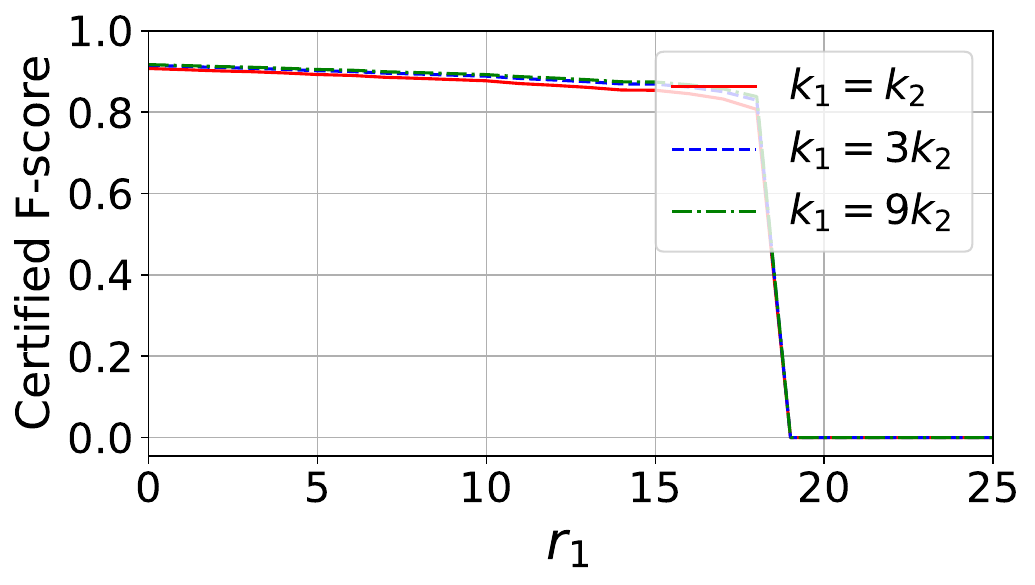}}
{\includegraphics[width=0.24\textwidth]{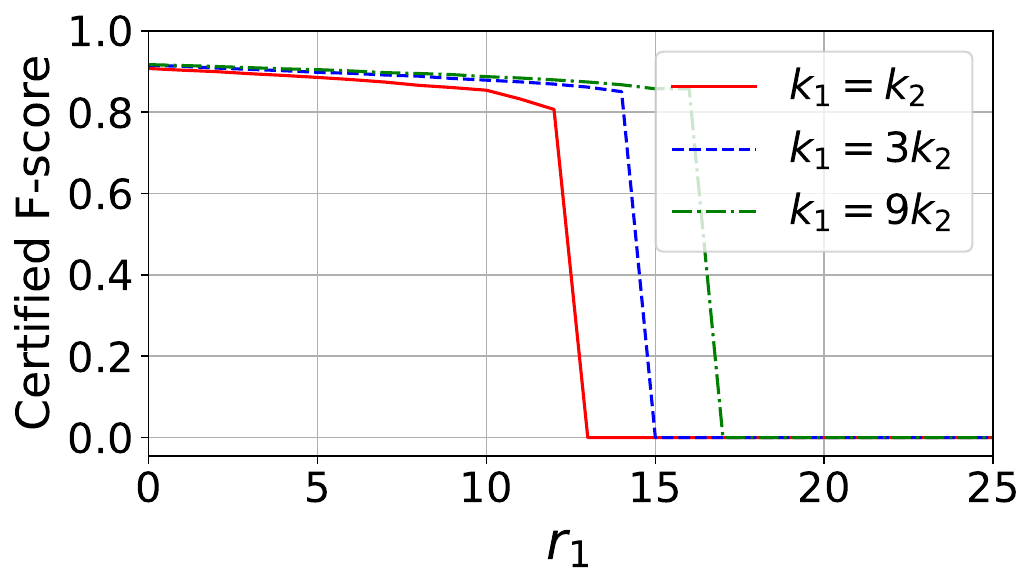}}
{\includegraphics[width=0.24\textwidth]{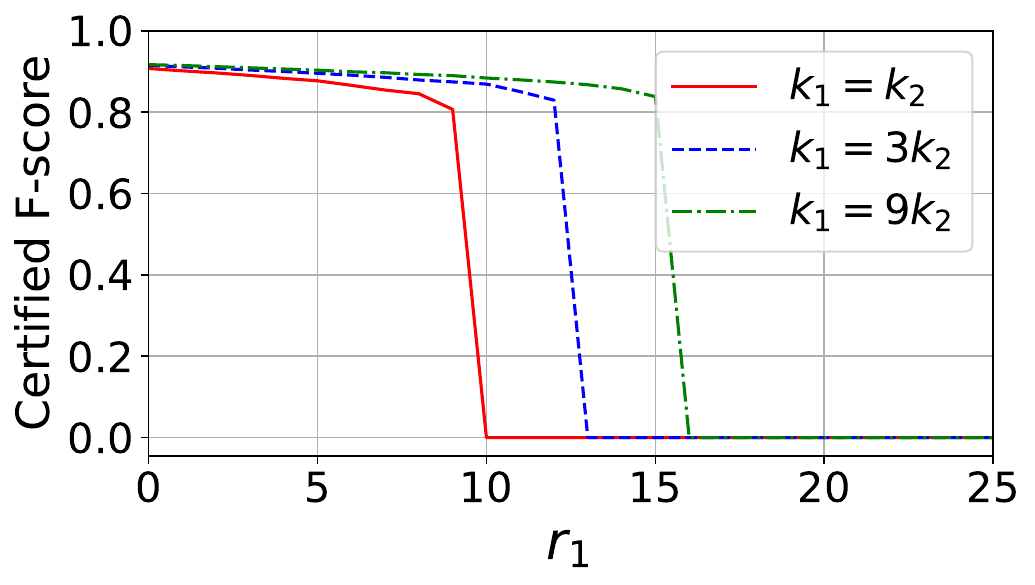}}
{\includegraphics[width=0.24\textwidth]{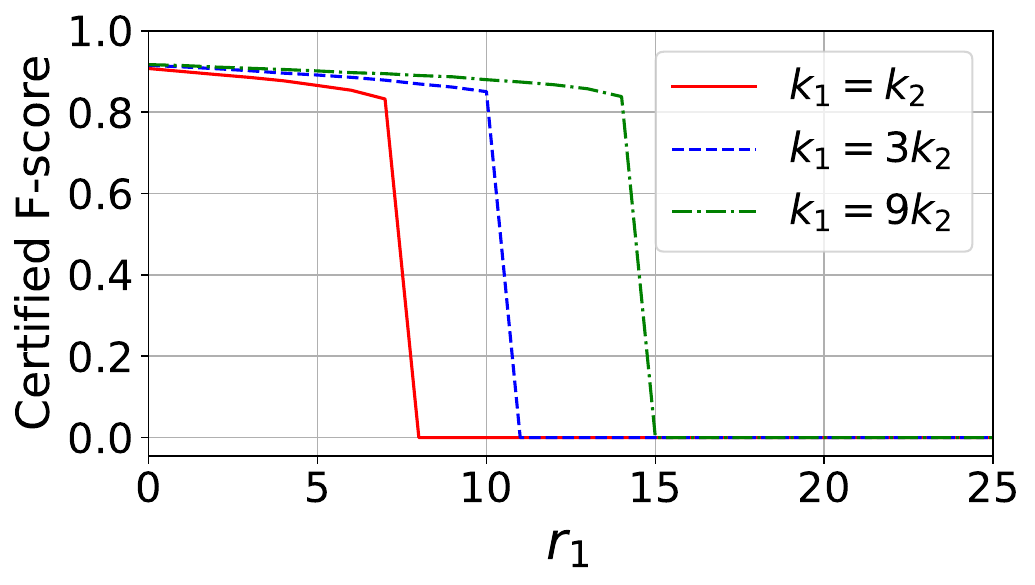}}
\subfloat[$r_2 = r_1$]{\includegraphics[width=0.24\textwidth]{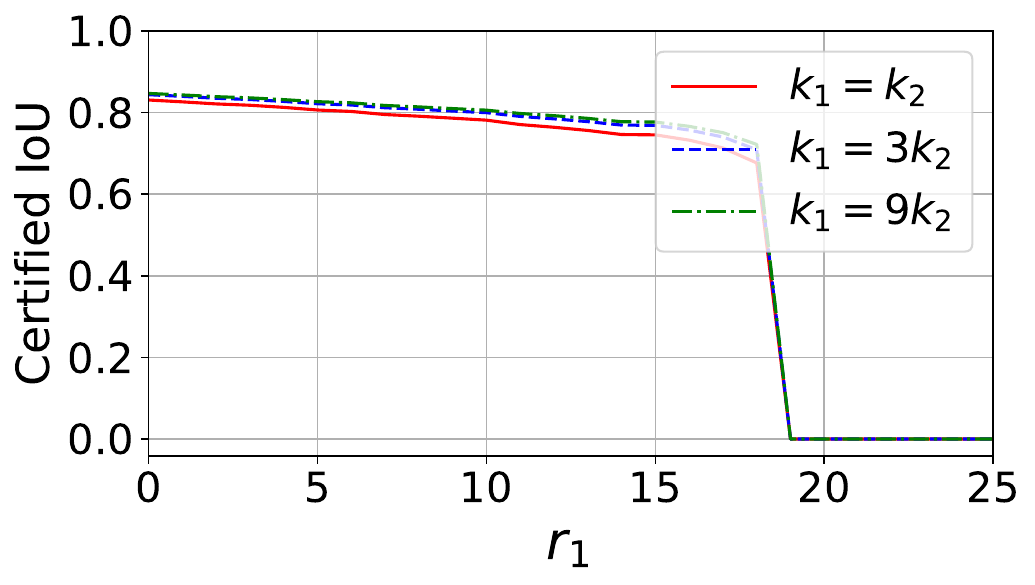}}\vspace{1mm}
\subfloat[$r_2 = 2r_1$]{\includegraphics[width=0.24\textwidth]{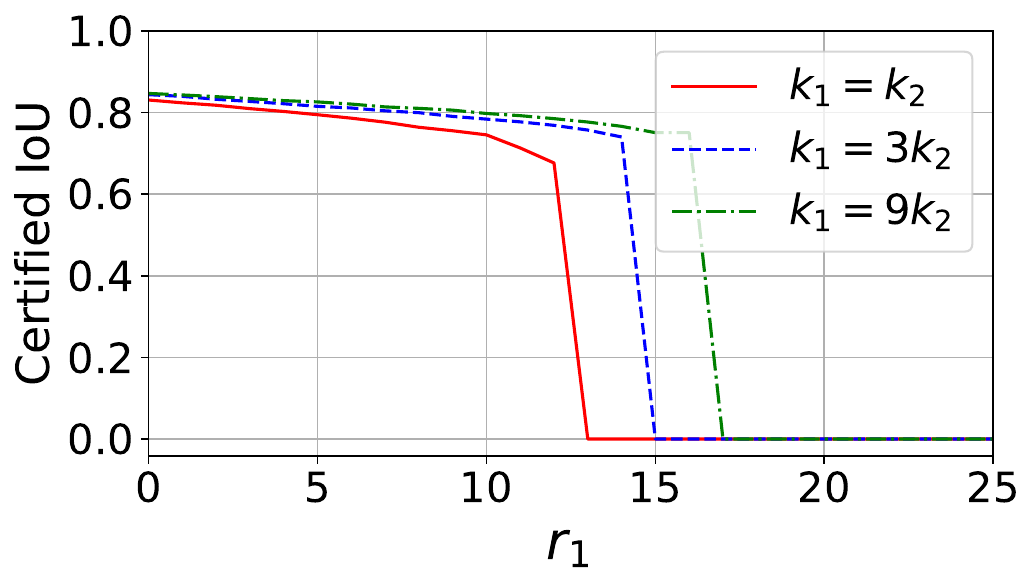}}\vspace{1mm}
\subfloat[$r_2 = 3r_1$]{\includegraphics[width=0.24\textwidth]{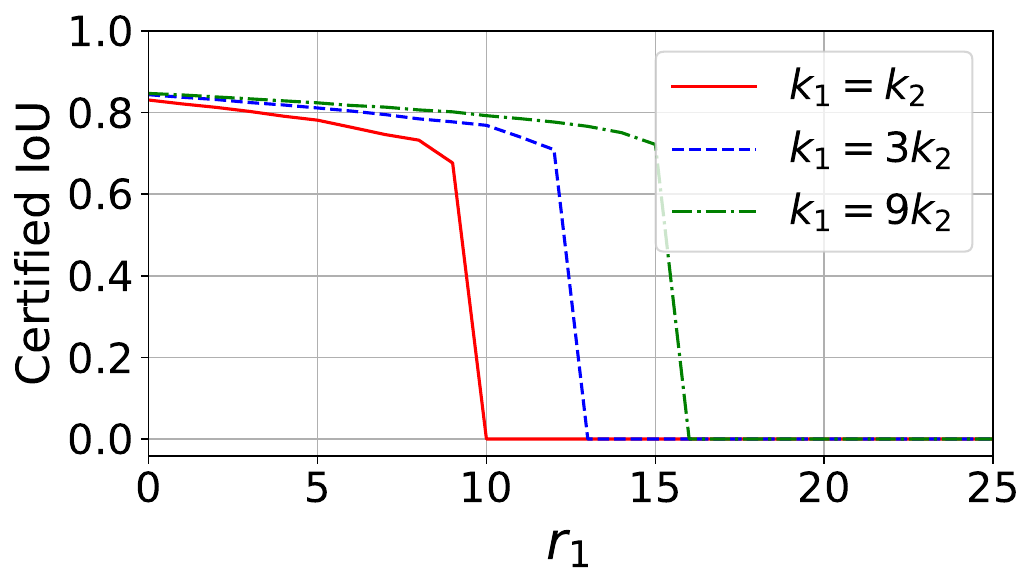}}\vspace{1mm}
\subfloat[$r_2 = 4r_1$]{\includegraphics[width=0.24\textwidth]{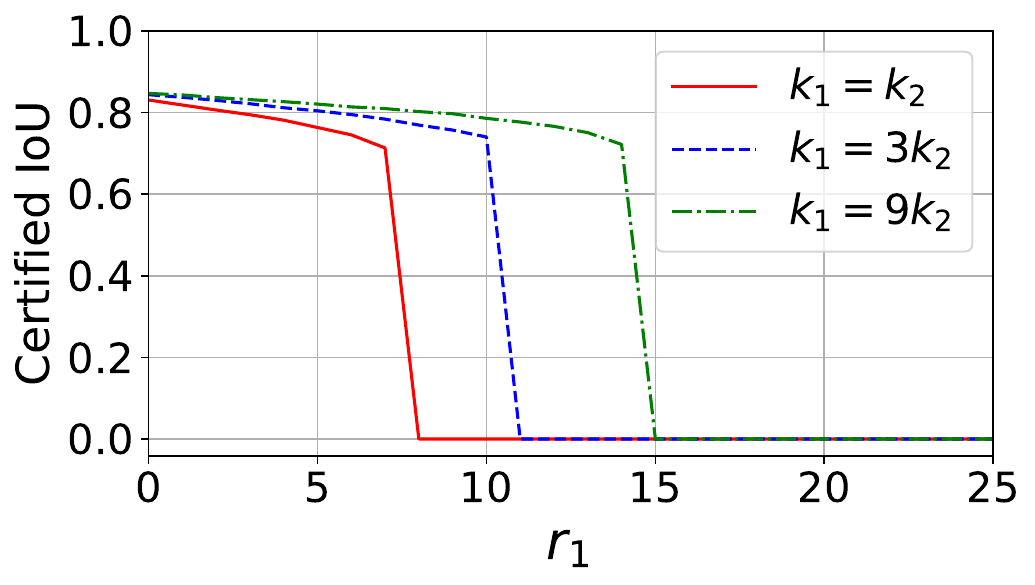}}\vspace{1mm}
\vspace{-4mm}
\caption{Impact of the ratio between $k_1$ and $k_2$. Certified Pixel Accuracy (first row), Certified F-score (second row) and Certified IoU (third row) are considered.
}
\label{exp-k-ratio}
\vspace{-7mm}
\end{figure*}

\section{Evaluation}\label{section-evaluation}
In this section, we demonstrate the effectiveness of our method on multi-modal emotion recognition task and multi-modal road segmentation task.
\subsection{Experimental Setup}
\label{sec:exp-setup}


\myparatight{Datasets} 
We use the following benchmark datasets in our evaluation: RAVDESS~\cite{livingstone2018ryerson}
for the multi-modal emotion recognition task and KITTI Road~\cite{KITTI-Road} for the multi-modal road segmentation task. Details of the datasets can be found in Appendix~\ref{appendix-datasets}.

 \myparatight{Models} For the multi-modal emotion recognition task, we follow the pipeline proposed by~\cite{chumachenko2022self}. Specifically, we utilize EfficientFace~\cite{zhao2021robust}
(a recently proposed facial expression recognition architecture) to extract features from
image frames, and use 1D
convolutional layers to extract features from
audio frames. Then we use intermediate attention-based fusion~\cite{chumachenko2022self} to combine features extracted from these two modalities. This fusion method ensures that features that are consistent between both modalities have the most significant impact on the final prediction.
 
 As for the multi-modal road segmentation task, we apply SNE-RoadSeg~\cite{fan2020sne}, which is capable of merging features from both RGB images and depth images for road segmentation. Specifically, this method first computes surface normal information from 
depth images, and then
employs a data-fusion CNN architecture to fuse features from both RGB images and the
inferred surface normal information for accurate prediction.

We note that if we directly use the original training recipes for these models, we get low prediction accuracy for randomly sub-sampled testing inputs, and the certified robustness of {\name} and randomized ablation~\cite{levine2020robustness} would be low as both rely on predicting ablated inputs. In response, we perform data augmentation by randomly ablate training inputs, such that the distribution of training data can match that of testing data. We note that this is standard practice for randomized smoothing-based certification methods~\cite{cohen2019certified}.
For {\name}, we independently sub-sample between 0\% and 5\% of basic elements from each modality, ablating the rest. For randomized ablation~\cite{levine2020robustness}, we randomly sample between 0\% and 5\% of basic elements collectively from the two modalities to keep and ablate the remaining elements. Specifically, we initially merge the basic elements of both modalities into one list. After sampling and ablating, we then split the modified list back into two separate modalities.

\myparatight{Compared Method} We compare our method with randomized ablation~\cite{levine2019sparse}, which is the state-of-the-art certification method for $l_0$ attacks on a single image. 

We adapt randomized ablation to multi-modal models by combining the sets of basic elements from each modality. Given the original input $(\mathbf{m}_1,\mathbf{m}_2,\ldots,\mathbf{m}_T)$, where $\mathbf{m}_i = [m^1_i,m^2_i, \ldots,m^{n_i}_i]$ 
, we combine all modalities to get $\mathbf{m} = [m^1_1,m^2_1, \ldots,m^{n_1}_1,\ldots, m^1_T,m^2_T, \ldots,m^{n_T}_T]$. We denote the size of $\mathbf{m}$ as $n = \sum_{i = 1}^T n_i$. Then we randomly sample $k$ elements from $\mathbf{m}$ without replacement to get a subset $\mathbf{z} \subseteq \mathbf{m}$. Finally, we divide $\mathbf{z}$ back to $T$ modalities, i.e., $\mathbf{z}_i = \mathbf{z} \cap \mathbf{m}_i$, and $(\mathbf{z}_1,\mathbf{z}_2,\ldots, \mathbf{z}_T)$ is the randomly ablated input. We make the final prediction by taking the majority vote of all ablated multi-modal inputs. 

For multi-modal classification tasks, we use the same certification process for randomized ablation as in the original paper~\cite{levine2019sparse}. For multi-modal segmentation tasks, we follow the same certification process as described in Section~\ref{sec-certify-segmentation} as the original work~\cite{levine2019sparse} only considered classification tasks. The only difference is that we define $\alpha_j^*$ as:
{\small
\begin{align}
&\alpha_j^* = \min_{\alpha_j}\\
s.t., \text{ }\text{ }&\underline{p_A}(\alpha_j)-1+ \frac{{n- \sum_{i = 1}^T r_i \choose k}}{{n\choose k}}
\geq
 \overline{p_B}(\alpha_j)+1- \frac{{n-\sum_{i = 1}^T r_i \choose k}}{{n\choose k}}.
\end{align}}
It is worth noting that in this context, $r_i$ represent the maximum number of modified basic elements in the $i$-th modality. The original work did not take into account addition and deletion attacks, as~\cite{levine2019sparse} focuses on image domain.

\myparatight{Parameter Settings} 
By default, we focus on modification attacks, where $r_i$ denote the maximum basic elements that can be modified by the attacker in $i$th modality.

For the multi-modal emotion recognition task, the visual modality contains 108 image frames, while the audio modality contains 79,380 audio frames. Without loss of generality, we denote the maximum number of modified image frames as $r_1$ and the maximum number of modified audio frames as $r_2$. The default setting is that the attacker can modify equal or more audio frames than image frames. That is, we let $r_2 = \hat{c} \cdot r_1$, for $\hat{c} = 1, 2, 3, 4$.  We set $k_1 = 5$ and $k_2 = 1,000$. For randomized ablation, $k$ is set to 3,000 such that, when there is no attack ($r_1 = r_2 = 0$), the accuracy of randomized ablation is similar to our {\name}. In Appendix~\ref{exp-r1-geq-r2}, we show the case where an attacker can modify more image frames than audio frames ($r_1 > r_2$).

For the multi-modal road segmentation task, the first modality is a RGB image that consists of $375\times 1,242$ pixels, where each pixel has three channels (representing the three primary colors), while the second modality is a depth image that has the same number of pixels, but each pixel has a single channel for depth. The default setting is that the attacker can modify equal or more pixels from the depth image than pixels from the RGB image. Specifically, we test for $r_2 = \hat{c} \cdot r_1$ where $\hat{c} = 1,2,3,4$.  For our {\name}, we set $k_1 = 9,000$ and $k_2 = 1,000$. 
Regarding randomized ablation, 
we set the total number of retained pixels $k$ to 10,000 
such when there is no attack ($r_1$ = $r_2$ = 0), the accuracy of randomized ablation is similar to our {\name}. In Appendix~\ref{exp-r1-geq-r2}, we show the case where the attacker can change more pixels from the RGB image than pixels from the depth image ($r_1>r_2$).

For Monte Carlo sampling, we set $N=100$ and $\alpha = 0.001$ for all experiments.

\myparatight{Evaluation Metrics}
We use \emph{Certified Accuracy} as the evaluation metric for the multi-modal emotion recognition task, and use \emph{Certified Pixel Accuracy}, \emph{Certified F-score}, and \emph{Certified IoU} as the evaluation metrics for the multi-modal road segmentation task. 

\begin{itemize}
    \item \myparatight{Certified Accuracy}Certified Accuracy is defined as the fraction of testing inputs whose predicted labels are not only correct but also verified to be unchanged by an attacker, i.e., certifiably stable. A testing sample for a multi-modal model can be represented as $(\mathbf{M}, y) \in D_{test}$, where $D_{test}$ is the testing dataset. $\mathbf{M}$ is the multi-modal test input and $y$ is the ground truth label. We use $G$ to denote the multi-modal classifier. Then we can define Certified Accuracy as:
    \begin{align}
        & \frac{\sum_{(\mathbf{M}, y) \in \mathcal{D}_{test}}\mathbb{I}(IsStable(\mathbf{M}) \wedge G(\mathbf{M}) = y)}{|\mathcal{D}_{test}|}.
    \end{align}
    $IsStable(\mathbf{M})$ is true if and only if for all $\mathbf{M}' \in \mathcal{S}(\mathbf{M}, \mathbf{R})$, we have $G(\mathbf{M}) = G(\mathbf{M}')$. $\mathbb{I}$ is the indicator function, and $|\mathcal{D}_{test}|$ is the total number of testing inputs in $\mathcal{D}_{test}$. 

    \item \myparatight{Certified Pixel Accuracy (or F-score or IoU)} Here we consider these certified metrics for the purpose of freespace detection~\cite{fan2020sne}. For general purposed segmentation tasks, mean values over different classes should be considered. The Certified Pixel Accuracy (or F-score or IoU) is defined as the average Pixel Accuracy (or F-score or IoU) lower bound of testing inputs under a given adversarial perturbation space $\mathbf{R}$. We use $j \in [n_o]$ to denote the index of a basic element of the segmented input modality $\mathbf{m}_o$. We define: $$TP = |\{j: (G_j(\mathbf{M})=y_j = 1) \wedge IsStable(\mathbf{M},j)\}|,$$ $$TN = |\{j: (G_j(\mathbf{M})=y_j = 0) \wedge IsStable(\mathbf{M},j)\}|,$$  $$FP = |\{j: G_j(\mathbf{M}) = 1\}| - TP, \text{and}$$ $$FN = |\{j: G_j(\mathbf{M}) = 0\}| - TN,$$ where label 1 represents freespace and label 0 represents non-freespace. $IsStable(\mathbf{M},j)$ is true if and only if the predicted label of the $j$th basic element of $\mathbf{m}_o$ cannot be changed by the attacker, i.e., $
    G_j(\mathbf{M}) = G_j(\mathbf{M}'), \forall \mathbf{M}' \in \mathcal{S}(\mathbf{M}, \mathbf{R})$. Then for an individual test sample, we have $\text{Certified Pixel Accuracy} = \frac{TP+TN}{TP+TN+FP+FN}$,
       $\text{Certified F-score} = \frac{2TP^2}{2TP^2+TP(FP+FN)}$
, and
       $\text{Certified IoU} = \frac{TP}{TP+FP+FN}$. 
To obtain the final metrics, we compute the average of these values across all test samples.
\end{itemize}
\subsection{Experimental Results}
In this section, we first compare our method with an existing state-of-the-art method, followed by an analysis of the impact of hyper-parameters on {\name}. Then, we show the performance of our method on attack types other than modification attack, i.e., addition and deletion attacks.

\myparatight{Our {\name} Outperforms Existing State-of-the-Art Method} Figure~\ref{exp-ravdess} and Figure~\ref{exp-kitti} show the comparison result between our {\name} and randomized ablation~\cite{levine2019sparse}, which is the state-of-the-art certified defense against $l_0$ attacks. We can see that our {\name} consistently outperforms randomized ablation on both tasks, for all combinations of $r_1$ and $r_2$
. For example, Figure~\ref{exp-ravdess} shows that on the RAVDESS dataset, when $r_1 = r_2 = 8$ (the attacker can modify 8 frames in both visual and audio modalities), our {\name} can guarantee correct predictions for more than $40\%$ of the test samples, while randomized ablation can guarantee $0\%$ of the test samples. 

Our method is more effective than randomized ablation because of two reasons. First, our method provides an adaptive selection of $k_1$ and $k_2$ to control the fraction of sub-sampled basic elements, i.e., $\frac{k_1}{n_1}$ and $\frac{k_2}{n_2}$, of the two modalities. In contrast, for randomized ablation, the sub-sampled fractions for both modalities are the identical on average, i.e., $\frac{k}{n}$. This means that randomized ablation is essentially a special case of our {\name}. Secondly, our {\name} is more stable than randomized ablation during both training and testing phases. In our method, the count of sub-sampled basic elements remains constant at $k_1$ and $k_2$ for each modality. Meanwhile, in randomized ablation, this count, adding up to $k$, fluctuates. As a result, our method's sub-sampled input space is smaller than that of randomized ablation, enhancing stability.

\myparatight{Impact of $k_1$ and $k_2$} 
Here we study the impact of $k_1$ and $k_2$ on the performance of our {\name}. To simplify the analysis, we perform the experiment on KITTI Road Dataset such that we have $n_1 = n_2$. This setup allows a direct comparison of the attacker's capability across two modalities using $r_1$ and $r_2$. We keep the sum of $k_1$ and $k_2$
  constant at 10,000 but vary their ratio. Three specific ratios were tested: $k_1 = k_2$, $k_1 = 3k_2$, and $k_1 = 9k_2$. The results are presented in Figure~\ref{exp-k-ratio}. We observe that with $r_2 = r_1$ (indicating similar attack capabilities on both modalities), different ratios of $k_1$ and $k_2$  have similar performance outcomes. However, for $r_2 > r_1$ (where the attacker has more attack capability on the second modality), strategies with a larger $k_1/k_2$
  ratio demonstrated better robustness. For example, if $r_2 >r_1$, the $k_2 = 9k_1$ sub-sampling strategy consistently outperforms $k_2 = 3k_1$, with this advantage magnifying as $r_2/r_1$ increased.. Therefore, in practice, it is advantageous to sub-sample fewer basic elements from the modality with higher attack capability and sub-sample more basic elements from the modality with lower attack capability, provided this doesn't compromise the utility (accuracy when there is no attack).
  
\myparatight{Different Attack Types}
We previously focused on modification attacks, where the attacker modifies at most $r_1$ and $r_2$ basic elements for the two respective modalities. Our method also allows the attacker to add or delete basic elements from each modality. Here we do experiments in scenarios where the attacker can add (or delete) at most $r_1$ and $r_2$ basic elements respectively for the two modalities, and compare with the modification attack scenario. The results are shown in Figure~\ref{exp-different_attacks}. We can see that modification attack is the strongest attack type. For example, when $r_1$ and $r_2$ are both 10, modification attack brings the certified accuracy down to 0. In contrast, the addition attack maintains a certified accuracy greater than 0.4, and the deletion attack maintains a certified accuracy greater than 0.6.

\vspace{-2mm}
\section{Conclusion}
\vspace{-2mm}
In this work, we propose {\name}, the first certified defense against adversarial attacks for multi-modal models. Our experimental results show that {\name} significantly improves the certified robustness guarantees by leveraging a modality-independent sub-sampling strategy.

\myparatight{Acknowledgements}
We thank the anonymous reviewers for their valuable feedback on our paper, which significantly improved the quality of our work.
{
    \small
    \bibliographystyle{ieeenat_fullname}
    \bibliography{egbib}
}

\appendix

\newpage
\onecolumn

\section{Proof of Theorem~\ref{theorem_of_certified_radius_classification}}
\label{proof_of_theorem_1}
Our proof is extended from previous studies~\cite{jia2020intrinsic}.
We first specify notations and then show our proof. 
Given original multi-modal input pair $\mathbf{M} = (\mathbf{m}_1,\mathbf{m}_2,\ldots, \mathbf{m}_T)$ and attacked input pair $(\mathbf{m}'_1,\mathbf{m}'_2,\ldots, \mathbf{m}'_T)$, we respectively use $\mathcal{X}$ and $\mathcal{Y}$ to denote the ablated multi-modal input sampled from them without replacement. We use $e_i$ to denote the number of basic elements (e.g., pixels) that are in both $\mathbf{m}_i$ and $\mathbf{m}'_i$, i.e., $e_i = |\mathbf{m}_i\cap \mathbf{m}'_i|$. Moreover, we use $ \Upsilon$ to denote the joint space between $\mathcal{X}$ and $\mathcal{Y}$.  We use $\mathcal{E} = (\mathcal{E}_1,\mathcal{E}_2, \ldots, \mathcal{E}_T)$ to denote a variable in the space $ \Upsilon$. 

We divide the space $ \Upsilon$ into the following subspace:
\begin{align}
    &\tilde{B} = \{\mathcal{E}| \mathcal{E}_1 \subseteq (\mathbf{m}_{1}\cap \mathbf{m}'_1),\mathcal{E}_2 \subseteq (\mathbf{m}_{2}\cap \mathbf{m}'_2),\ldots, \mathcal{E}_T \subseteq (\mathbf{m}_{T}\cap \mathbf{m}'_T)\}, \\  
    &\tilde{A} = \{\mathcal{E}|\mathcal{E}_1 \subseteq \mathbf{m}_{1},\mathcal{E}_2 \subseteq \mathbf{m}_{2}, \ldots, \mathcal{E}_T \subseteq \mathbf{m}_{T}\} - \tilde{B},\\
    &\tilde{C} = \{\mathcal{E}|\mathcal{E}_1 \subseteq \mathbf{m}'_{1},\mathcal{E}_2 \subseteq \mathbf{m}'_{2}, \ldots, \mathcal{E}_T \subseteq \mathbf{m}'_{T}\} - \tilde{B}.
\end{align}
We present Neyman Pearson Lemma~\cite{neyman1933ix,cohen2019certified,jia2020intrinsic} for later use.
\begin{lemma}[Neyman Pearson]
\label{lemma_np_general}
Let $\mathcal{X}$, $\mathcal{Y}$ be two random variables whose probability densities are respectively $\text{Pr}(\mathcal{X}=\mathcal{E})$ and $\text{Pr}(\mathcal{Y}=\mathcal{E})$, where $\mathcal{E}\in  \Upsilon$. Let $Z$ be a random or deterministic functions. 
where $Z(1|\mathcal{E})$ denotes the probability that $Z(\mathcal{E})=1$. 
Then, we have the following:  

(1) If $W_1=\{\mathcal{E}\in  \Upsilon: \text{Pr}(\mathcal{Y}=\mathcal{E})/\text{Pr}(\mathcal{X}=\mathcal{E}) < \mu  \}$ and $W_2=\{\mathcal{E}\in  \Upsilon: \text{Pr}(\mathcal{Y}=\mathcal{E})/\text{Pr}(\mathcal{X}=\mathcal{E}) = \mu  \}$ for some $\mu > 0$. Let $S = W_1 \cup W_3$, where $W_3 \subseteq W_2$. If $\text{Pr}(Z(\mathcal{X})=1)\geq \text{Pr}(\mathcal{X}\in S)$, then $\text{Pr}(Z(\mathcal{Y})=1)\geq \text{Pr}(\mathcal{Y}\in S)$. 

(2) If $W_1=\{\mathcal{E}\in  \Upsilon: \text{Pr}(\mathcal{Y}=\mathcal{E})/\text{Pr}(\mathcal{X}=\mathcal{E}) > \mu  \}$ and $W_2=\{\mathcal{E}\in  \Upsilon: \text{Pr}(\mathcal{Y}=\mathcal{E})/\text{Pr}(\mathcal{X}=\mathcal{E}) = \mu  \}$ for some $\mu > 0$. Let $S = W_1 \cup W_3$, where $W_3 \subseteq W_2$. If $\text{Pr}(Z(\mathcal{X})=1)\leq \text{Pr}(\mathcal{X}\in S)$, then $\text{Pr}(Z(\mathcal{Y})=1)\leq \text{Pr}(\mathcal{Y}\in S)$. 
\end{lemma}
\begin{proof}
Let's start by proving part (1). For convenience, we denote the complement of $S$ as $S^{c}$. With this notation, we have the following: 
\begin{align}
  & \text{Pr}(Z(\mathcal{Y})=1)- \text{Pr}(\mathcal{Y}\in S) \\
=&\int_{ \Upsilon}Z(1|\mathcal{E})\cdot \text{Pr}(\mathcal{Y}=\mathcal{E})d\mathcal{E} - \int_{S}\text{Pr}(\mathcal{Y}=\mathcal{E})d\mathcal{E} \\
=&\int_{S^{c}}Z(1|\mathcal{E})\cdot \text{Pr}(\mathcal{Y}=\mathcal{E})d\mathcal{E} + 
\int_{S}Z(1|\mathcal{E})\cdot \text{Pr}(\mathcal{Y}=\mathcal{E})d\mathcal{E} - \int_{S}\text{Pr}(\mathcal{Y}=\mathcal{E})d\mathcal{E} \\
=&\int_{S^{c}}Z(1|\mathcal{E})\cdot \text{Pr}(\mathcal{Y}=\mathcal{E})d\mathcal{E} 
\label{lemma_np_general_e1}
- \int_{S}(1-Z(1|\mathcal{E}))\cdot\text{Pr}(\mathcal{Y}=\mathcal{E})d\mathcal{E} \\
\label{lemma_np_general_e2}
\geq&\mu \cdot[\int_{S^{c}}Z(1|\mathcal{E})\cdot \text{Pr}(\mathcal{X}=\mathcal{E})d\mathcal{E} 
- \int_{S}(1-Z(1|\mathcal{E}))\cdot\text{Pr}(\mathcal{X}=\mathcal{E})d\mathcal{E}] \\
=&\mu \cdot[\int_{S^{c}}Z(1|\mathcal{E})\cdot \text{Pr}(\mathcal{X}=\mathcal{E})d\mathcal{E} + 
\int_{S}Z(1|\mathcal{E})\cdot \text{Pr}(\mathcal{X}=\mathcal{E})d\mathcal{E} - \int_{S}\text{Pr}(\mathcal{X}=\mathcal{E})d\mathcal{E}] \\
=&\mu \cdot[\int_{ \Upsilon}Z(1|\mathcal{E})\cdot \text{Pr}(\mathcal{X}=\mathcal{E})d\mathcal{E} - \int_{S}\text{Pr}(\mathcal{X}=\mathcal{E})d\mathcal{E}] \\
=&\mu \cdot [\text{Pr}(Z(\mathcal{X})=1)- \text{Pr}(\mathcal{X}\in S)] \\
\geq & 0.
\end{align}
Equation~\ref{lemma_np_general_e2} is derived from~\ref{lemma_np_general_e1} due to the fact that $\text{Pr}(\mathcal{Y}=\mathcal{E})/\text{Pr}(\mathcal{X}=\mathcal{E}) \leq \mu, \forall \mathcal{E}\in S$, $\text{Pr}(\mathcal{Y}=\mathcal{E})/\text{Pr}(\mathcal{X}=\mathcal{E}) \geq \mu, \forall \mathcal{E}\in S^{c}$, and  $1-Z(1|\mathcal{E}) \geq 0$. 
Similarly, we can establish the proof for part (2), but we have omitted the detailed steps for the sake of conciseness.
\end{proof}

For simplicity, we use $n_i$ and $n_i'$ to denote the number of basic elements (e.g., pixels) in $\mathbf{m}_i$ and $\mathbf{m}'_i$ respectively, i.e., $n_i=|\mathbf{m}_i|$ and $n_i' = |\mathbf{m}_i'|$.
Then, we have the following probability mass function:
\begin{align}
&\text{Pr}(\mathcal{X} = \mathcal{E}) = 
\begin{cases}
 \frac{1}{\prod_{i=1}^T{n_i \choose k_i}}, &\text{ if } \mathcal{E} \in \tilde{A}  \cup \tilde{B}, \\
 0, &\text{ otherwise}.
\end{cases} \\
&\text{Pr}(\mathcal{Y} = \mathcal{E}) = 
\begin{cases}
 \frac{1}{\prod_{i=1}^T{n'_i \choose k_i}}, &\text{ if } \mathcal{E} \in \tilde{B}  \cup \tilde{C}, \\
 0, &\text{ otherwise}.
\end{cases}
\end{align}

Recall that we have $e_i=|\mathbf{m}'_{i}\cap \mathbf{m}_{i}|$ for $i = 1,2, \ldots, T$, so the probability of $\mathcal{X}$ and $\mathcal{Y}$ in $\tilde{A}$, $\tilde{B}$ and $\tilde{C}$ can be computed as follows:
\begin{align}
   & \text{Pr}(\mathcal{X} \in \tilde{A}) =1- \frac{\prod_{i=1}^T{e_i \choose k_i}}{\prod_{i=1}^T{n_i \choose k_i}}, \text{Pr}(\mathcal{X} \in \tilde{B}) = \frac{\prod_{i=1}^T{e_i \choose k_i}}{\prod_{i=1}^T{n_i \choose k_i}}, \text{Pr}(\mathcal{X} \in \tilde{C}) = 0; \\
   & \text{Pr}(\mathcal{Y} \in \tilde{A}) =0, \text{Pr}(\mathcal{Y} \in \tilde{B}) = \frac{\prod_{i=1}^T{e_i \choose k_i}}{\prod_{i=1}^T{n'_i \choose k_i}}, \text{Pr}(\mathcal{Y} \in \tilde{C}) = 1- \frac{\prod_{i=1}^T{e_i \choose k_i}}{\prod_{i=1}^T{n'_i \choose k_i}}.
\end{align}

We first define $\delta_{l}=\underline{\Pr}(g(\mathcal{X})=A) - \frac{\lfloor\underline{\Pr}(g(\mathcal{X})=A)\prod_{i=1}^T{n_i \choose k_i}\rfloor}{\prod_{i=1}^T{n_i \choose k_i}} $ to help rounding $\underline{\Pr}(g(\mathcal{X})=A)$. Then we can construct a set $S  = \tilde{A} + \tilde{B}'$, where $\tilde{B}'\subseteq \tilde{B}$ and $\Pr(\mathcal{X}\in \tilde{B}') = \underline{\Pr}(g(\mathcal{X}) = A)-\delta_{l}-\Pr(\mathcal{X}\in \tilde{A})$. We can assume $\underline{\Pr}(g(\mathcal{X}) = A) >\Pr(\mathcal{X}\in \tilde{A})$ because otherwise $\Pr(g(\mathcal{Y}) = A)$ is bounded by 0. Then we have ${\Pr}(g(\mathcal{X}) = A)\geq\Pr(\mathcal{X}\in S)$. So we have the following lower bound on $\Pr(g(\mathcal{Y}) = A)$:

\begin{align}
&\Pr(g(\mathcal{Y}) = A)\\
\geq &\Pr(\mathcal{Y} \in S)\\
\geq &\Pr(\mathcal{Y} \in \tilde{B}')\\
\geq &\Pr(\mathcal{X} \in \tilde{B}')\frac{\Pr(\mathcal{Y} \in \tilde{B}')}{\Pr(\mathcal{X} \in \tilde{B}')}\\
\geq &\frac{\prod_{i=1}^T{n_i \choose k_i}}{\prod_{i=1}^T{n'_i \choose k_i}}(\underline{\Pr}(g(\mathcal{X}) = A)-\delta_{l}-1+ \frac{\prod_{i=1}^T{e_i \choose k_i}}{\prod_{i=1}^T{n_i \choose k_i}})
\end{align}

Similarly we define $\delta_{u}=\frac{\lceil\overline{\Pr}(g(\mathcal{X})=B)\prod_{i=1}^T{n_i \choose k_i}\rceil}{\prod_{i=1}^T{n_i\choose k_i}} -\overline{\Pr}(g(\mathcal{X})=B) $, so we can construct a set $S  = \tilde{B}'+\tilde{C}$, where $\tilde{B}'\subseteq \tilde{B}$ and $\Pr(\mathcal{X}\in \tilde{B}') = \overline{\Pr}(g(\mathcal{X}) = B)+ \delta_{u}-\Pr(\mathcal{X}\in \tilde{C})$. Then we have ${\Pr}(g(\mathcal{X}) = B)\leq\Pr(\mathcal{X}\in S)$. So we have the following upper bound on $\Pr(g(\mathcal{Y}) = B)$:

\begin{align}
&\Pr(g(\mathcal{Y}) = B)\\
\leq &\Pr(\mathcal{Y} \in S)\\
\leq &\Pr(\mathcal{Y} \in \tilde{B}')+\Pr(\mathcal{Y} \in \tilde{C})\\
\leq &\Pr(\mathcal{X} \in \tilde{B}')\frac{\Pr(\mathcal{Y} \in \tilde{B}')}{\Pr(\mathcal{X} \in \tilde{B}')}+\Pr(\mathcal{Y} \in \tilde{C})\\
\leq &\frac{\prod_{i=1}^T{n_i \choose k_i}}{\prod_{i=1}^T{n'_i \choose k_i}}(\overline{\Pr}(g(\mathcal{X}) = B)+\delta_{u})+\Pr(\mathcal{Y} \in \tilde{C})\\
\leq &\frac{\prod_{i = 1}^T{n_i \choose k_i}}{\prod_{i=1}^T{n'_i\choose k_i}}(\overline{\Pr}(g(\mathcal{X}) = B)+\delta_{u})+1- \frac{\prod_{i = 1}^T{e_i \choose k_i}}{\prod_{i = 1}^T{n'_i \choose k_i}}
\end{align}
To certify a test sample, we just need to enforce $\Pr(g(\mathcal{Y}) = A)>\Pr(g(\mathcal{Y}) = B)$. So we get Theorem~\ref{theorem_of_certified_radius_classification}.

\section{Details About the Datasets}\label{appendix-datasets}
We use two benchmark datasets for evaluation.
\begin{itemize}
\item \myparatight{RAVDESS} We use RAVDESS dataset~\cite{livingstone2018ryerson}
for the multi-modal emotion recognition task. This dataset contains video recordings of
24 participants, each speaking with a variety of emotions. The goal is to classify these emotions into one of seven categories: calm, happy,
sad, angry, fearful, surprise, and disgust. For each participant, there are 60 distinct video sequences. For data pre-processing, we follow previous work~\cite{chumachenko2022self, multimodal-emotion-recognition} and crop or zero-pad these videos to
3.6 seconds, which is the average video length. After pre-processing, each data sample contains 108 image frames and 79380 audio frames. We assume that the attacker can arbitrarily modify $r_1$ image frames (from 108 image frames of visual input) and $r_2$ audio frames (from 79380 audio frames of audio input). We divide  the data into training,
validation and test sets ensuring that the identities of actors
are not repeated across sets. Particularly, we used four actors
for testing, four for validation, and the remaining 16 for training. 

\item \myparatight{KITTI Road} For the multi-modal road segmentation task, we use KITTI Road Dataset~\cite{KITTI-Road}, which contains 289 training and 290 test samples across three distinct road scene categories. Notably, the initial release~\cite{KITTI-Road} lacks ground-truth labels for its test samples. As a result, we divided the original training dataset into 231 data samples (80\% of the data samples) for training and 58 data samples (20\% of the data samples) for testing. Each data sample consists of a RGB image, a depth image, and the ground truth segmentation. We assume that the attacker can arbitrarily modify $r_1$ pixels from the RGB image and $r_2$ pixels from the depth image for each testing input.
\end{itemize}
\section{Special Cases in Multi-modal Segmentation}\label{appendix-segmentation}
For segmentation tasks, the multi-modal model outputs the segmentation result for one of the input modalities $\mathbf{m}_o$ with $n_o$ basic elements, which can be pixels or 3-D points. Then the output contains $n_o$ labels. Previously, we consider the case where the attacker perform modification attacks to $\mathbf{m}_o$, where we have $\mathbf{m}_o = n_o = n'_o = |\mathbf{m}'_o|$. However, deletion and addition attacks on $\mathbf{m}_o$ are also possible if $\mathbf{m}_o$ represents a point cloud. If that is the case, the process of deriving Certified Pixel Accuracy, Certified F-score and Certified IoU can be different. 

First, we think of the multi-modal segmentation model before the attack (denoted by $G$) as composed of multiple classifiers denoted by $G_1,G_2,\ldots, G_{n_o}$. Each classifier $G_j$ predicts a label $G_j(\mathbf{M})$ for $m^j_o$ (the $j$th basic element of $\mathbf{m}_o$). The ground truth $y$ also includes $n_o$ labels, denoted by $y_1,y_2,\ldots, y_{n_o}$. We use $G_j(\mathbf{M})$ to denote the predicted label for $m^j_o$ before the attack and use $G_j(\mathbf{M}')$ to denote the predicted label for $m^j_o$ after the attack. We say a basic element (e.g., a pixel) $m^j_o$ is \emph{certifiably stable} if 
\begin{align}
    &G_j(\mathbf{M}) = G_j(\mathbf{M}')
    , \forall \mathbf{M}' \in \mathcal{S}(\mathbf{M}, \mathbf{R}), \text{and } m^j_o \in \mathbf{m}_o \cap \mathbf{m}'_o,\end{align}
which means $j$th basic element of $\mathbf{m}_o$ is also in $\mathbf{m}'_o$ and the predicted label for it is unchanged by the attack. If it also holds that $G_j(\mathbf{M}) = y_j$, then we term $m^j_o$ as \emph{certifiably robust}.

Then we derive Certified Pixel Accuracy (or F-score or IoU) for deletion and addition attacks on $\mathbf{m}_o$. We use $j \in [n_o]$ to denote the index of a basic element of the input modality $\mathbf{m}_o$. For each label, we define: $$TP = |\{j: (G_j(\mathbf{M})=y_j = 1) \wedge IsStable(\mathbf{M},j)\}|,$$ $$TN = |\{j: (G_j(\mathbf{M})=y_j = 0) \wedge IsStable(\mathbf{M},j)\}|,$$  $$FP = |\{j: G_j(\mathbf{M}) = 1\}| - TP, \text{and}$$ $$FN = |\{j: G_j(\mathbf{M}) = 0\}| - TN,$$ where 1 indicates that this basic element has been identified as belonging to this label, while label 0 signifies the opposite. $IsStable(\mathbf{M},j)$ is true if and only if the $j$th basic element of $\mathbf{m}_o$ is certifiably stable as defined above. We use $r_o$ denote the added (or deleted) basic elements for $\mathbf{m}_o$. Then for addition attacks to $\mathbf{m}_o$, the worst case is that all added basic elements are not certifiably robust, so we have $\text{Certified Pixel Accuracy} = \frac{TP+TN}{TP+TN+FP+FN+r_o}$,
       $\text{Certified F-score} = \frac{2TP^2}{2TP^2+TP(FP+FN+r_o)}$
, and
       $\text{Certified IoU} = \frac{TP}{TP+FP+FN+r_o}$. And for deletion attacks to $\mathbf{m}_o$, the worst case is that all deleted basic elements are certifiably robust, so we have $\text{Certified Pixel Accuracy} = \frac{TP+TN-r_o}{TP+TN+FP+FN-r_o}$,
       $\text{Certified F-score} = \frac{2(TP-r_o)^2}{2(TP-r_o)^2+(TP-r_o)(FP+FN)}$
, and
       $\text{Certified IoU} = \frac{TP-r_o}{TP+FP+FN-r_o}$. 
To obtain the final metrics, we compute the average of these values across all test samples and all labels.
\begin{figure*}
\centering
\subfloat[$r_1 = 2r_2$]{\includegraphics[width=0.3\textwidth]{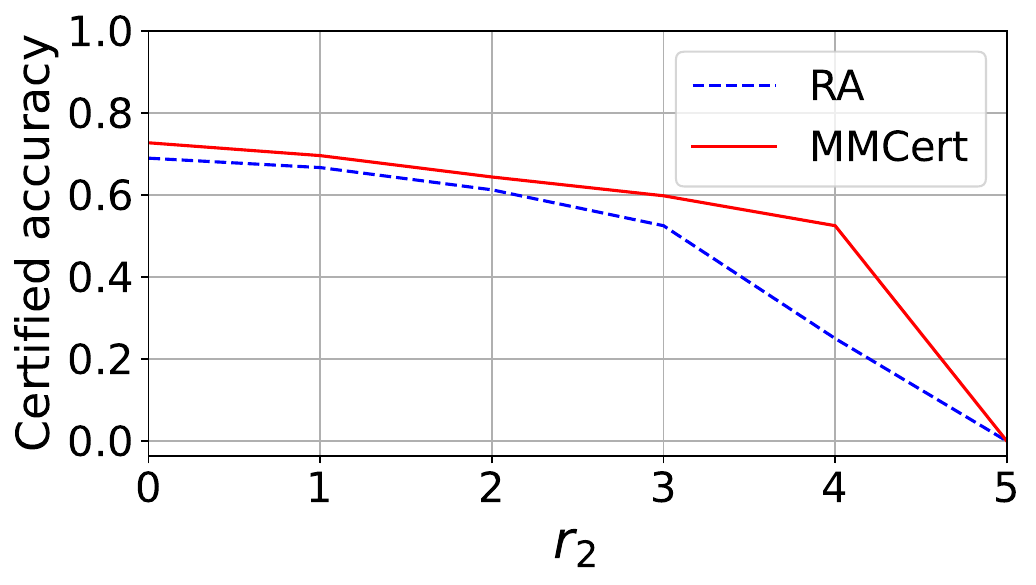}}\vspace{1mm}
\subfloat[$r_1 = 3r_2$]{\includegraphics[width=0.3\textwidth]{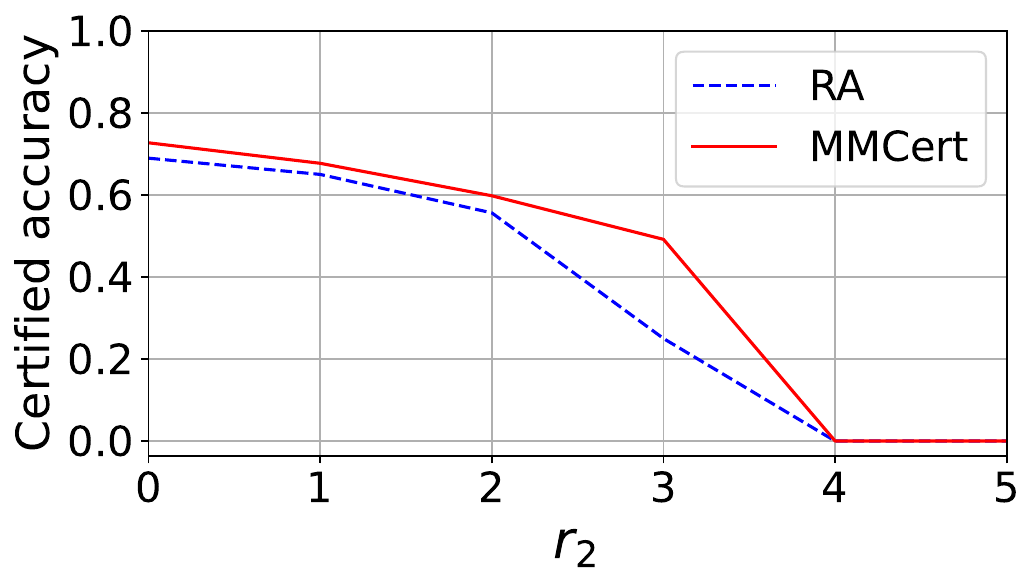}}\vspace{1mm}
\subfloat[$r_1 = 4r_2$]{\includegraphics[width=0.3\textwidth]{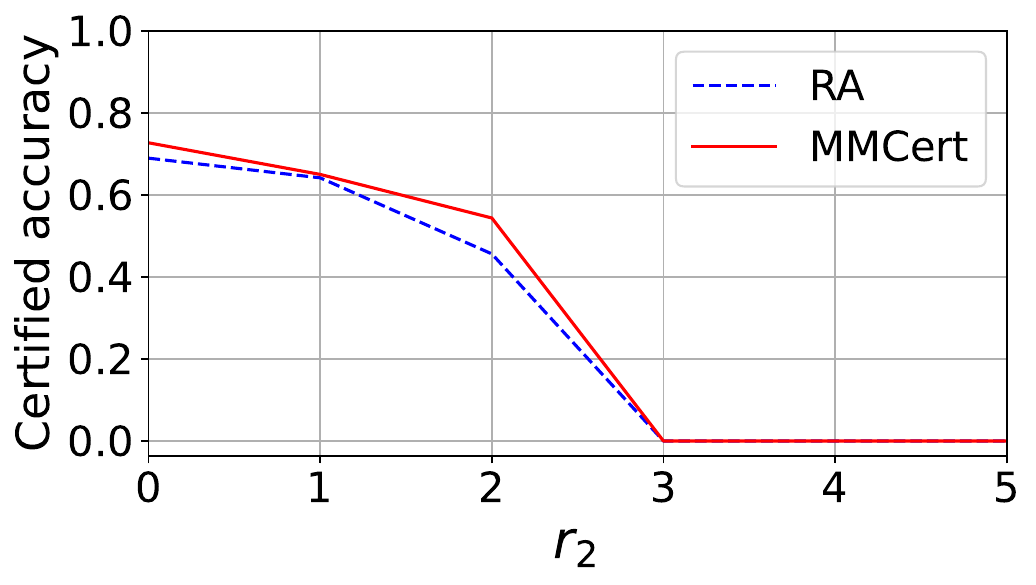}}\vspace{1mm}
\vspace{-4mm}
\caption{Compare our {\name} with randomized ablation on RAVDESS Dataset.
}
\label{exp-r1geqr2-ravdess}

\end{figure*}

\begin{figure*}
\centering

{\includegraphics[width=0.3\textwidth]{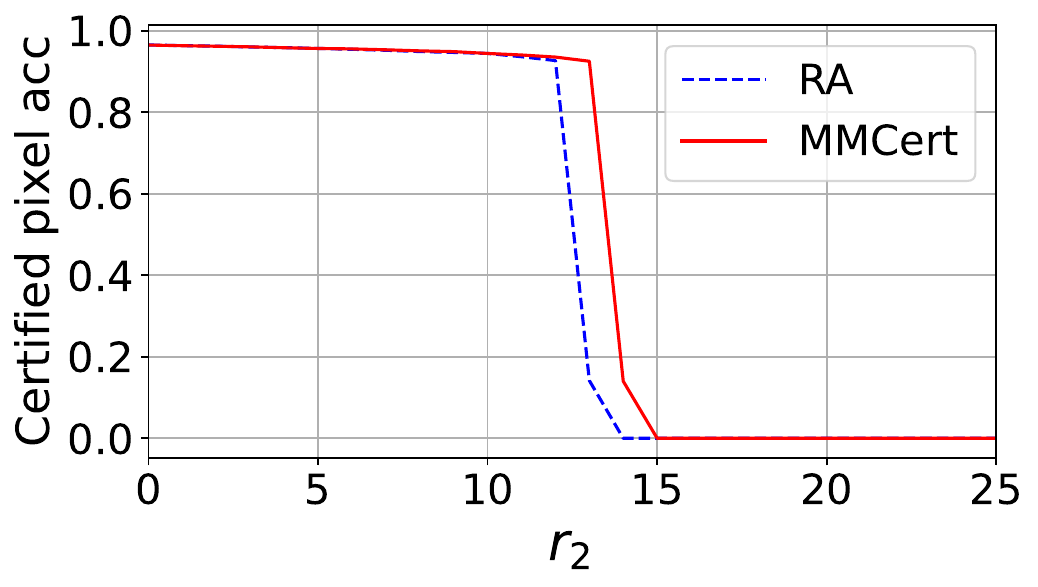}}\vspace{1mm}
{\includegraphics[width=0.3\textwidth]{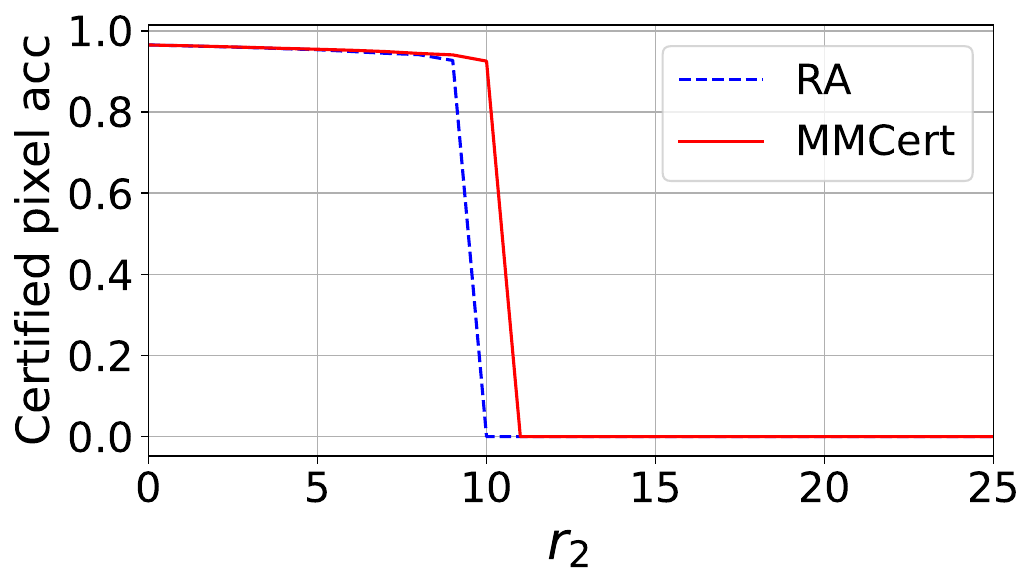}}\vspace{1mm}
{\includegraphics[width=0.3\textwidth]{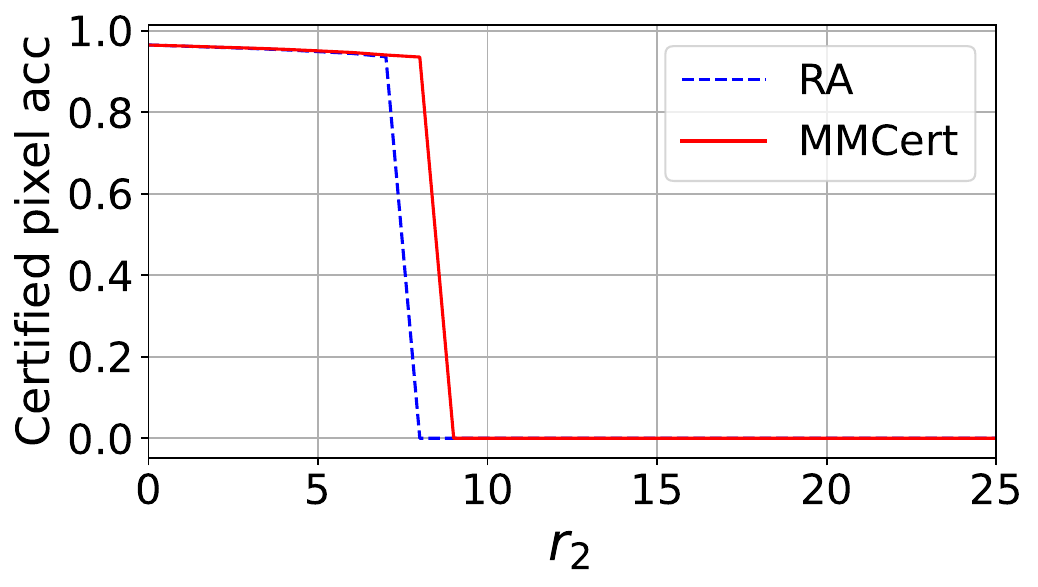}}\vspace{1mm}
{\includegraphics[width=0.3\textwidth]{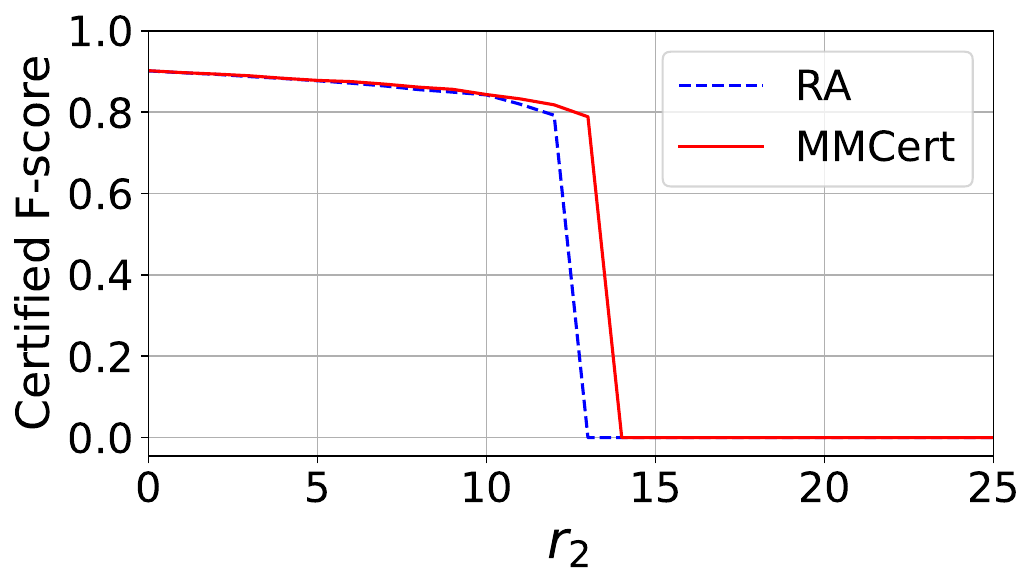}}\vspace{1mm}
{\includegraphics[width=0.3\textwidth]{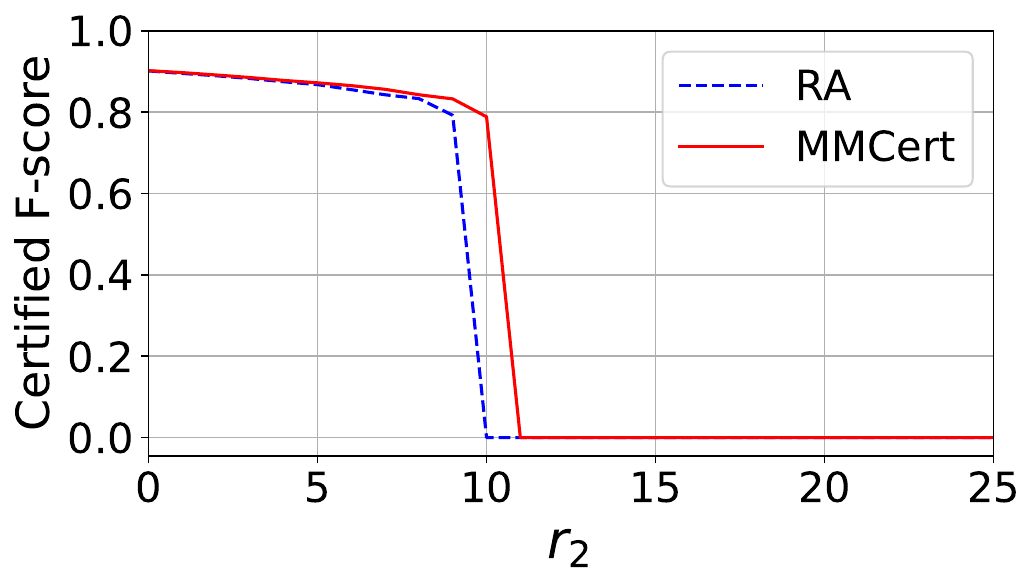}}\vspace{1mm}
{\includegraphics[width=0.3\textwidth]{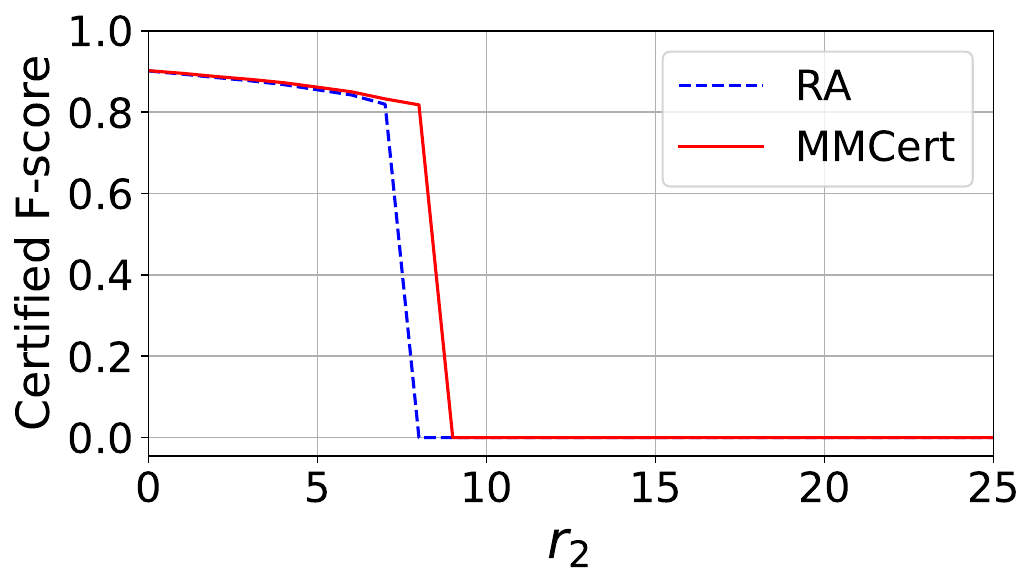}}\vspace{1mm}
\subfloat[$r_1 = 2r_2$]{\includegraphics[width=0.3\textwidth]{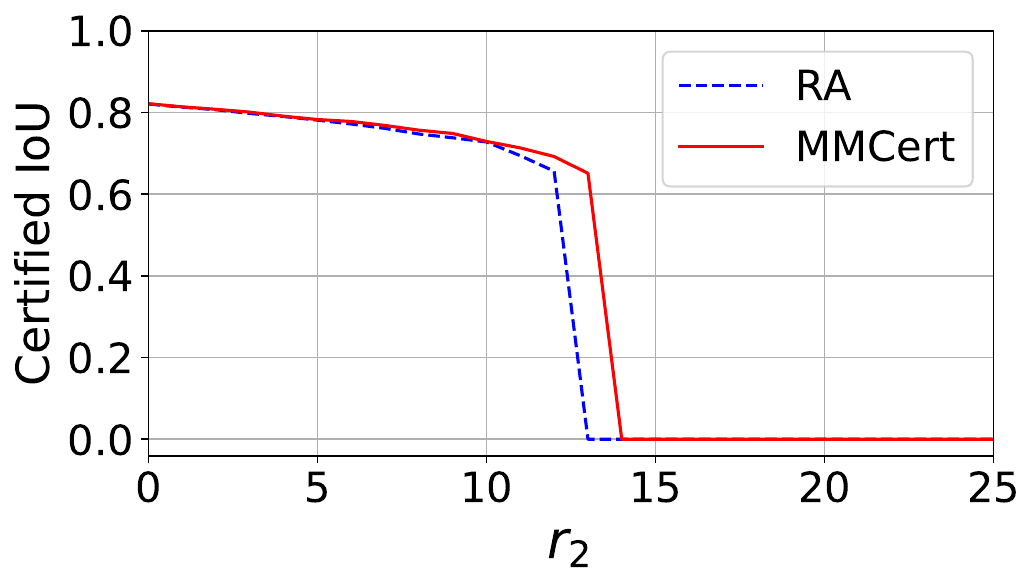}}\vspace{1mm}
\subfloat[$r_1 = 3r_2$]{\includegraphics[width=0.3\textwidth]{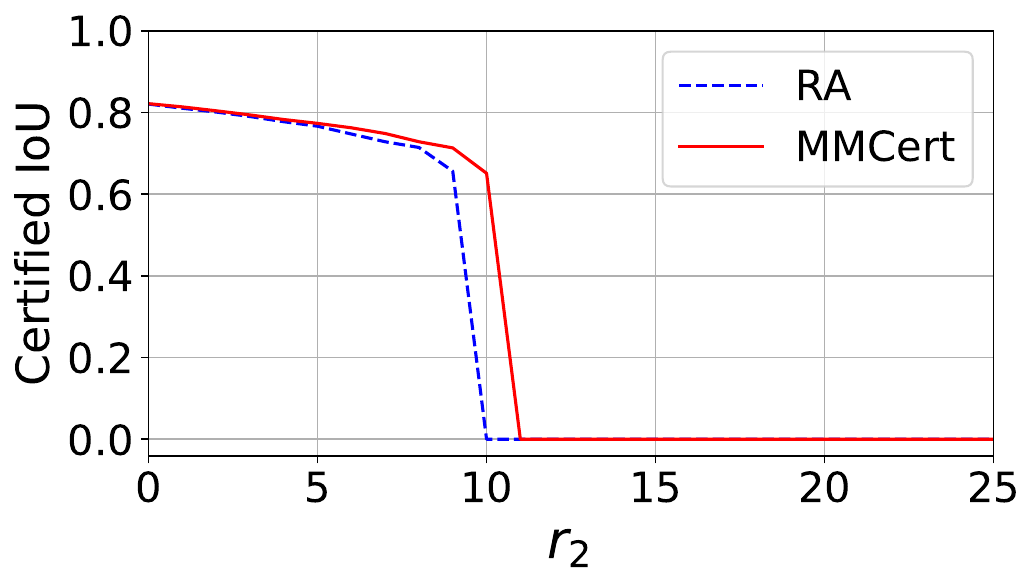}}\vspace{1mm}
\subfloat[$r_1 = 4r_2$]{\includegraphics[width=0.3\textwidth]{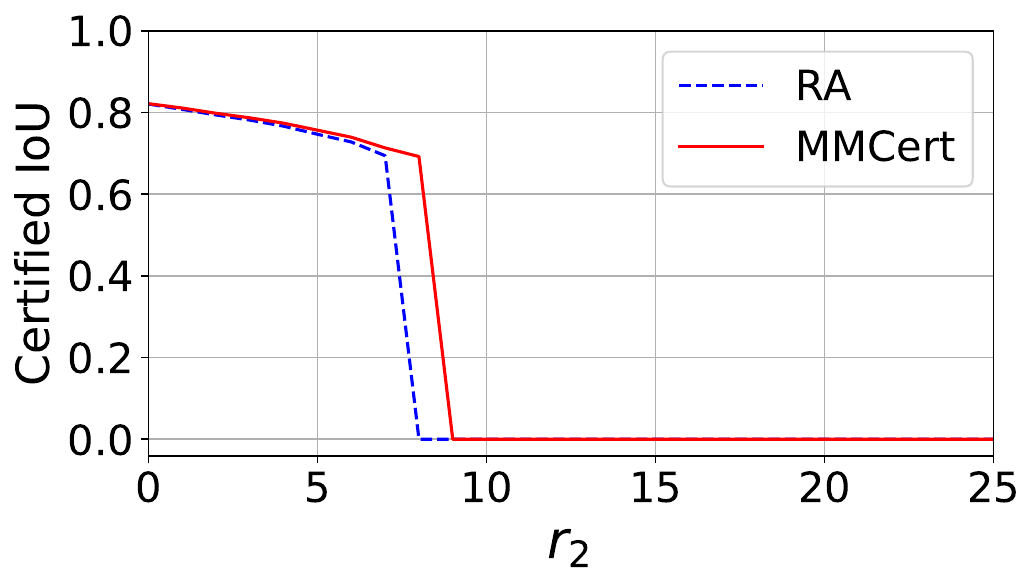}}\vspace{1mm}
\vspace{-4mm}
\caption{Compare our {\name} with randomized ablation on KITTI Road Dataset. Certified Pixel Accuracy (first row), Certified F-score (second row) and Certified IoU (third row) are considered.
}
\label{exp-r1geqr2-kitti}

\end{figure*}

\section{Experiment Results for the $r_1>r_2$ Case}\label{exp-r1-geq-r2}
Here, we compare our method with randomized ablation for the case $r_1>r_2$. For KITTI Road dataset, we set $k_1$ to 4,000 and $k_2$ to 6,000 for our {\name} and set $k$ to 10,000 for randomized ablation. For RAVEDESS, we let $k_1$ = 5 and $k_2$ = 1,000 for our {\name} and let $k$ = 3,000 for randomized ablation. The results of these experiments are illustrated in Figures \ref{exp-r1geqr2-ravdess} and \ref{exp-r1geqr2-kitti}, corresponding to RAVNESS and KITTI Road datasets, respectively. Our findings reveal that our method consistently surpasses randomized ablation across all $r_1$-$r_2$ ratios for both datasets. This can be attributed to the fact that randomized ablation is essentially a special case of our {\name}. Consequently, we can identify a combination of $k_1$ and $k_2$ that yields equal or better results than randomized ablation. Furthermore, our {\name} is more stable than randomized ablation during both training and testing phases because our method's sub-sampled input space is smaller than that of randomized ablation.

\begin{figure*}
\centering

\subfloat[$r_1 = r_2$]{\includegraphics[width=0.24\textwidth]{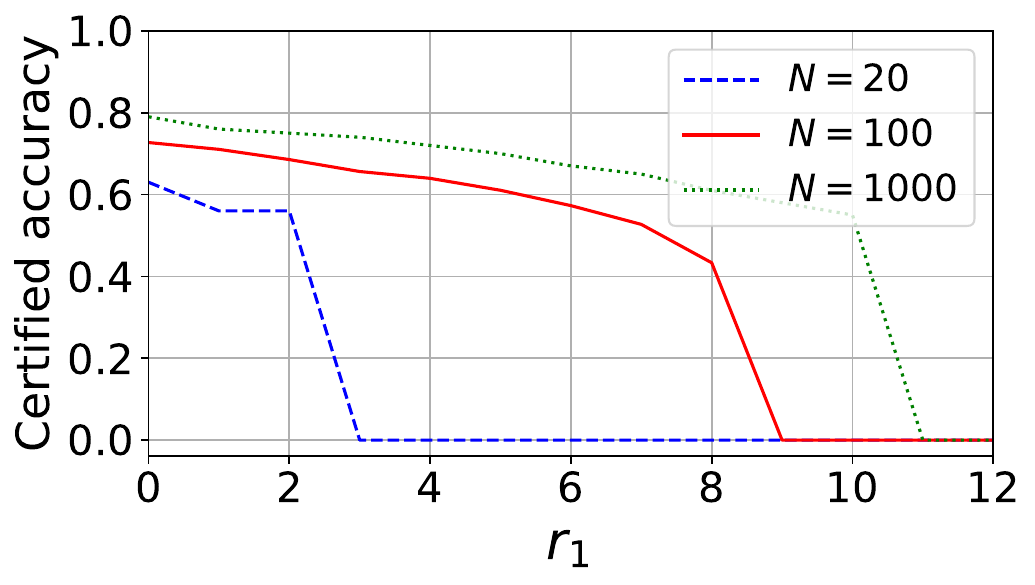}}\vspace{1mm}
\subfloat[$r_1 = 2r_2$]{\includegraphics[width=0.24\textwidth]{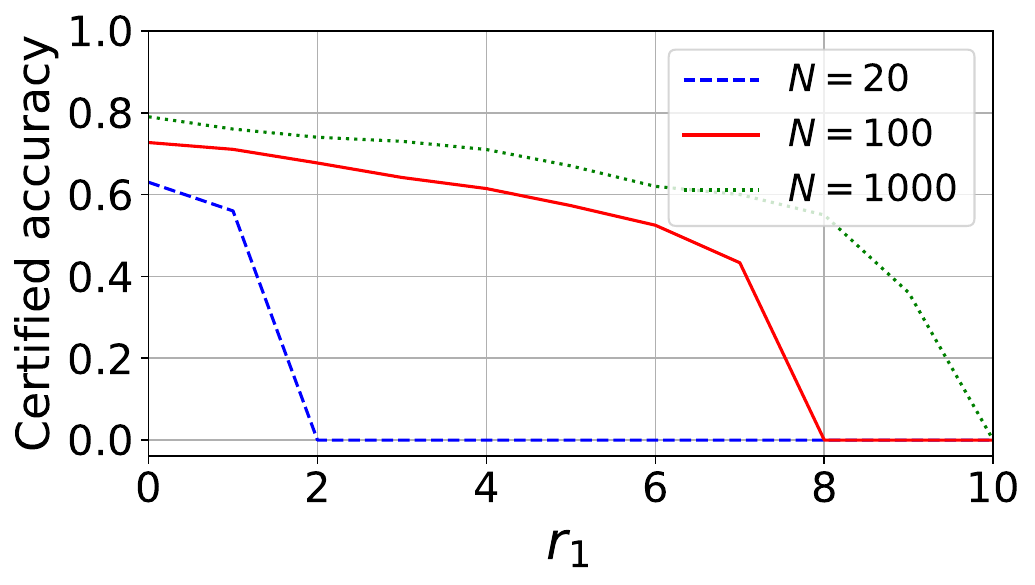}}\vspace{1mm}
\subfloat[$r_1 = 3r_2$]{\includegraphics[width=0.24\textwidth]{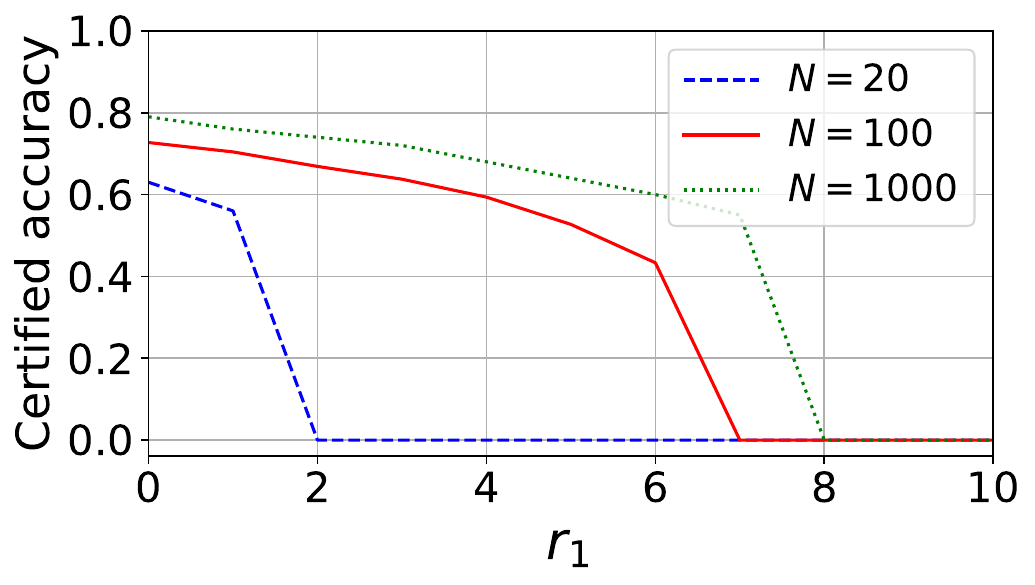}}\vspace{1mm}
\subfloat[$r_1 = 4r_2$]{\includegraphics[width=0.24\textwidth]{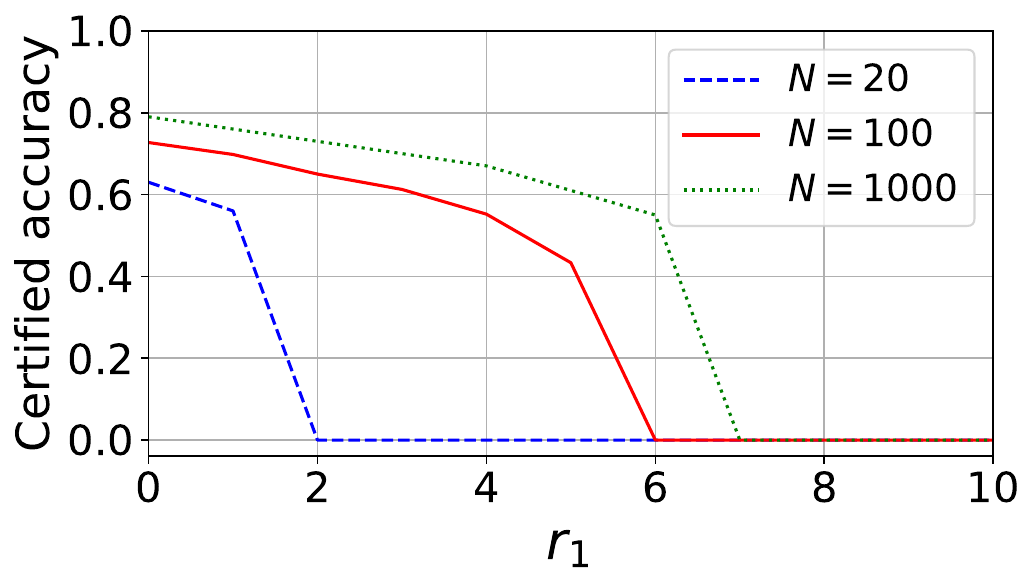}}\vspace{1mm}

\vspace{-4mm}
\caption{Impact of $N$ on RAVDESS dataset.
}
\label{exp-ablation-N}

\end{figure*}
\begin{figure*}
\centering

\subfloat[$r_1 = r_2$]{\includegraphics[width=0.24\textwidth]{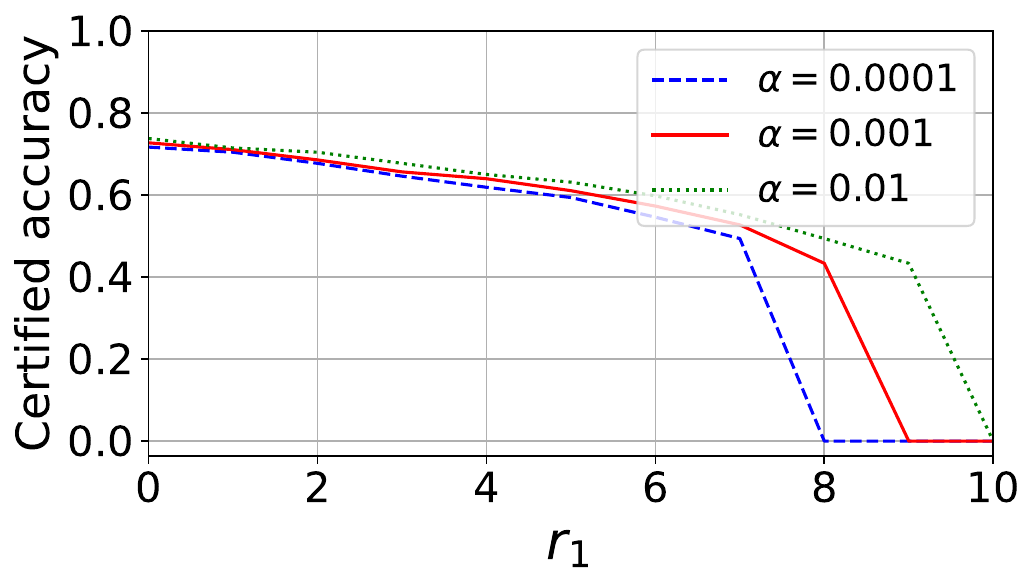}}\vspace{1mm}
\subfloat[$r_1 = 2r_2$]{\includegraphics[width=0.24\textwidth]{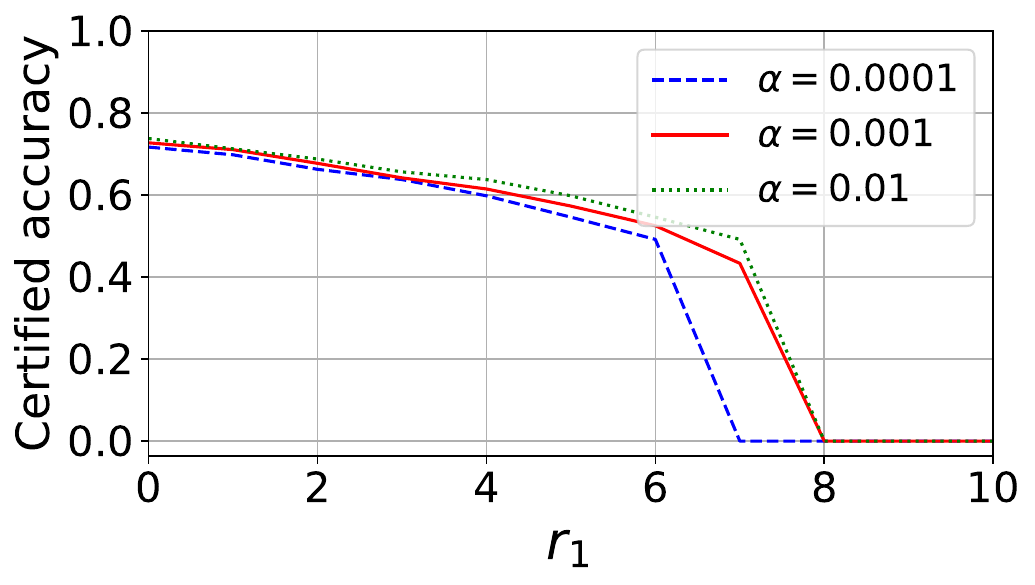}}\vspace{1mm}
\subfloat[$r_1 = 3r_2$]{\includegraphics[width=0.24\textwidth]{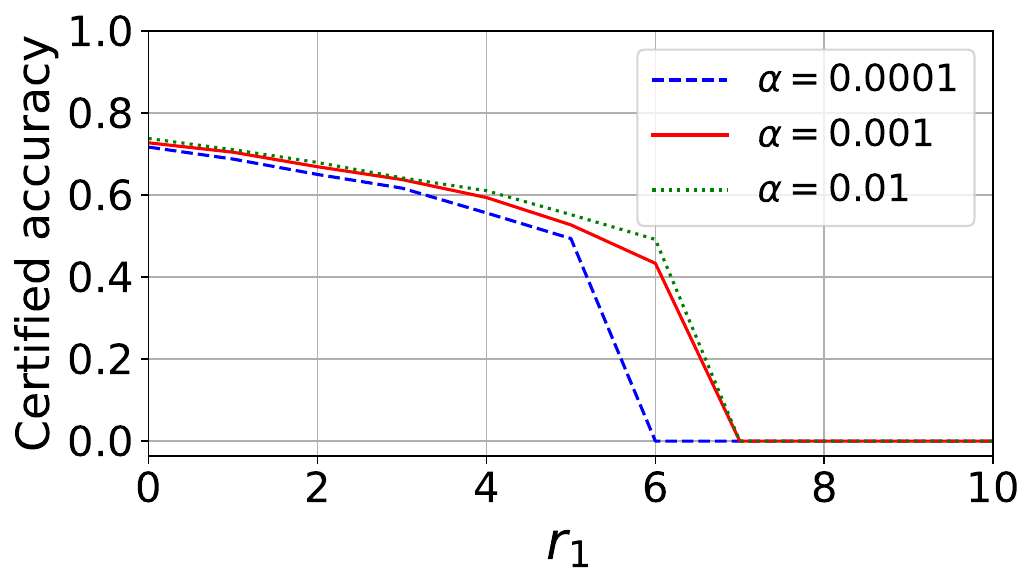}}\vspace{1mm}
\subfloat[$r_1 = 4r_2$]{\includegraphics[width=0.24\textwidth]{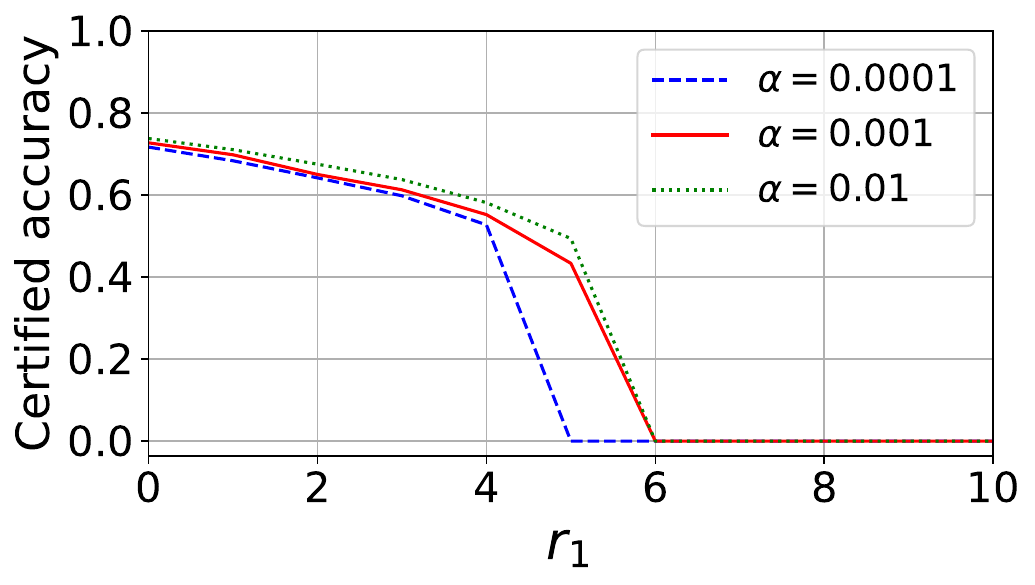}}\vspace{1mm}

\vspace{-4mm}
\caption{Impact of $\alpha$ on RAVDESS dataset.
}
\label{exp-ablation-alpha}

\end{figure*}
\section{Impact of $N$ and $\alpha$}
We study the impact of $N$ and $\alpha$ on RAVDESS dataset. Figure~\ref{exp-ablation-N} in Appendix shows the impact of $N$.  We discover that the certified accuracy improves with an increase in $N$. This enhancement occurs because a larger $N$ yields tighter lower or upper bounds for the label probability, given a constant confidence level $\alpha$. However, the computational cost also grows linearly with respect to $N$, reflecting a trade off between computational cost and certification performance. Figure~\ref{exp-ablation-alpha} in Appendix shows the impact of $\alpha$. We observe that {\name} achieves
better performance as $\alpha$ increases. This shows the trade off between the confidence of the certification and the certification performance.

\begin{figure*}
\centering
{\includegraphics[width=1\textwidth]{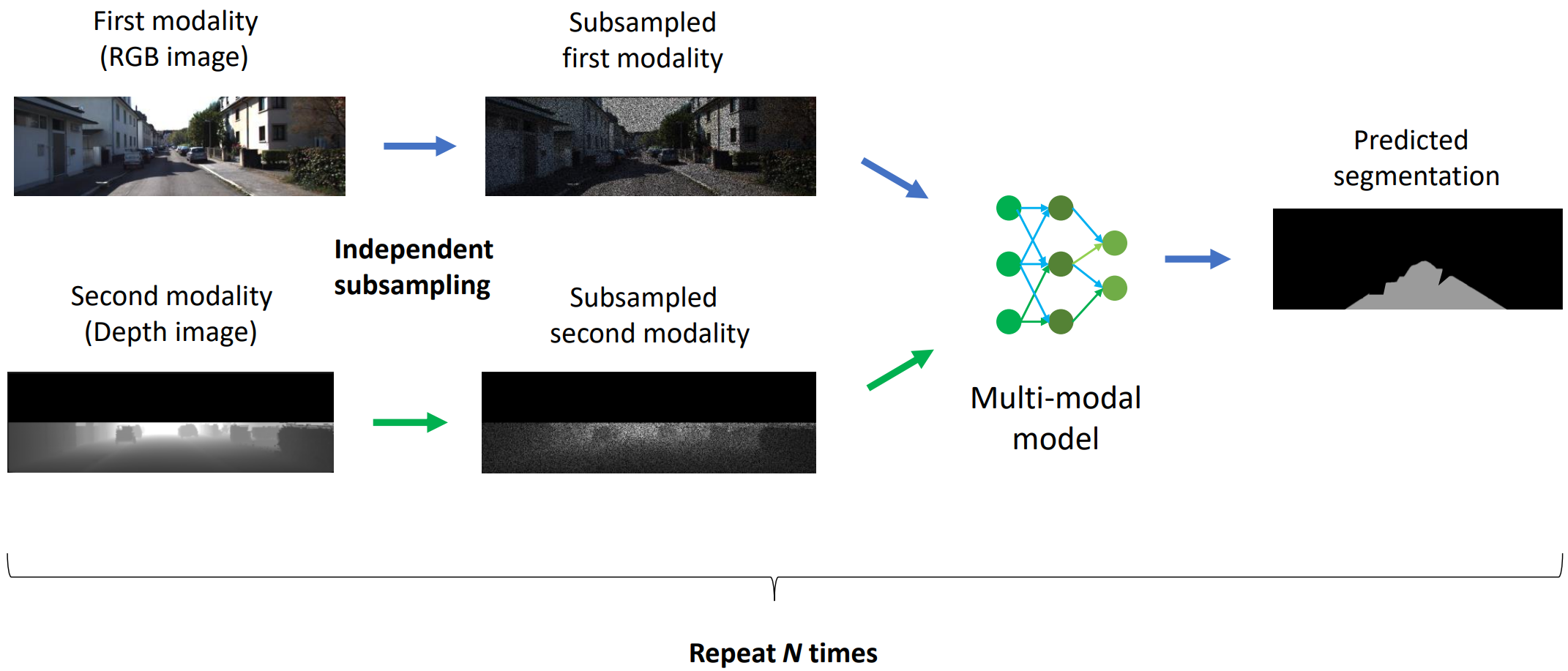}}\vspace{1mm}
\caption{Illustration of independent sub-sampling on KITTI Road dataset. Our method repeatedly generate predictions for subsampled multi-modal inputs. These predictions are then aggregated to get the final prediction.}
\label{fig-subsampling}
\end{figure*}
\end{document}